\documentclass{article}

\usepackage[preprint]{neurips_2020}

\usepackage[utf8]{inputenc} 
\usepackage[T1]{fontenc}    
\usepackage{hyperref}       
\usepackage{url}            
\usepackage{booktabs}       
\usepackage{amsfonts}       
\usepackage{nicefrac}       
\usepackage{microtype}      
\usepackage{hyperref}
\usepackage{algorithm}
\usepackage{algpseudocode}
\usepackage{graphicx}
\usepackage{amssymb,amsthm,amsfonts,amscd,mathrsfs}
\usepackage{thm-restate}
\usepackage{amsmath}
\usepackage{cleveref}

\newtheorem{thm}{Theorem}
\newtheorem{lma}[thm]{Lemma}

\title{Training End-to-End Analog Neural Networks\\
with Equilibrium Propagation}

\author{%
  Jack Kendall$^1$, Ross Pantone$^1$, Kalpana Manickavasagam$^1$,\\
  \textbf{Yoshua Bengio$^{2,3}$, Benjamin Scellier$^{2,}\thanks{Currently at Google.}$}\\
  $^1$Rain Neuromorphics\\
  $^2$Mila, Université de Montréal\\
  $^3$Canadian Institute for Advanced Research 
}

\begin{document}

\maketitle

\begin{abstract}
We introduce a principled method to train end-to-end analog neural networks by stochastic gradient descent. In these analog neural networks, the weights to be adjusted are implemented by the conductances of programmable resistive devices such as memristors \citep{chua1971memristor}, and the nonlinear transfer functions (or `activation functions') are implemented by nonlinear components such as diodes. We show mathematically that a class of analog neural networks (called nonlinear resistive networks) are energy-based models: they possess an energy function as a consequence of Kirchhoff's laws governing electrical circuits. This property enables us to train them using the Equilibrium Propagation framework \citep{Scellier+Bengio-frontiers2017}. Our update rule for each conductance, which is local and relies solely on the voltage drop across the corresponding resistor, is shown to compute the gradient of the loss function. Our numerical simulations, which use the SPICE-based \textit{Spectre} simulation framework to simulate the dynamics of electrical circuits, demonstrate training on the MNIST classification task, performing comparably or better than equivalent-size software-based neural networks. Our work can guide the development of a new generation of ultra-fast, compact and low-power neural networks supporting on-chip learning.
\end{abstract}

\section{Introduction}

In recent years, deep neural networks have proved extremely effective in machine learning, achieving state-of-the-art performance in a variety of domains, including image classification \citep{he2016deep}, speech recognition \citep{hinton2012deep}, machine translation \citep{vaswani2017attention}, and text-to-speech \citep{oord2016wavenet}. One of the core principles to train these deep neural networks is optimization by stochastic gradient descent (SGD).

However, training these neural networks on graphics processing units (GPUs) is time consuming and energy intensive.
This is due to the separation of memory and processing in von Neumann hardware, which leads to a severe bottleneck in moving the data back and forth between memory and compute units -- the so-called \textit{von Neumann bottleneck}. Building fast and energy-efficient neural networks requires a non-von Neumann computing paradigm which unifies memory and processing, by performing neural computations at the physical location of the synapses, where the strength of the connections (the weights of the neural network) are stored and adjusted.

Programmable resistors can implement the synapses of a neural network by encoding the synaptic weights in their conductances. Such programmable resistors can be built into large crossbar arrays to represent the weight matrices of the layer-to-layer transformations of a deep neural network. This setup presents significant advantages over multi-processors such as GPUs: the number of operations that can be simultaneously executed in a GPU is limited by its number of processors ; in contrast, a crossbar array is a massively-parallel computing device in which all resistors do parallel computations. Furthermore, since the computation in a crossbar array is in the analog domain, the power consumption is also several orders of magnitude lower than that of a GPU \citep{burr2017neuromorphic}.

Crossbar arrays have been proposed to accelerate the hardware implementation of the backpropagation algorithm, which is the key algorithm of deep learning to train conventional neural networks \citep{li2018efficient,rekhi2019ams}. Nevertheless, these implementations require digital-to-analog (DAC) and analog-to-digital (ADC) conversion between the layers of the network, the latter being known to be responsible for most of the power consumption \citep{li2015merging}. This suggests an opportunity to further cut power consumption by avoiding DACs and ADCs.

In this work, we suggest an alternative to the backpropagation algorithm, which eliminates the need for DACs and ADCs. We introduce end-to-end analog neural networks, in which nonlinear resistive components such as diodes play the role of nonlinear transfer functions. Crucially, our hardware implementation of neural networks allows end-to-end analog computing not just for inference, but also for training. To achieve this, we use the Equilibrium Propagation (or EqProp) framework, suitable for optimization by SGD with local weight updates \citep{Scellier+Bengio-frontiers2017}.

EqProp applies to energy-based models (EBMs), i.e. neural network models that possess an \textit{energy function}. A key result we show is that a class of analog neural networks called \textit{nonlinear resistive networks} are EBMs: they possess an energy function whose existence is a direct consequence of Kirchhoff's laws governing electrical circuits.
As a consequence, these analog networks are trainable by SGD using locally available information for each weight. Specifically, we show mathematically that the gradient (of the loss to minimize) with respect to a conductance can be estimated using solely the voltage drop across the corresponding resistor. This result opens up a path towards dramatically faster and orders of magnitude higher energy-efficient neural networks trainable by SGD.

The main contributions of the present work are the following.
\begin{itemize}
\item Inspired by the work of \citet{johnson2010nonlinear}, we show that a class of analog neural networks called \textit{nonlinear resistive networks} are energy-based models (EBMs): at inference, the configuration of node voltages chosen by the circuit corresponds to the minimum of a mathematical function (the \textit{energy function}) called the \textit{total pseudo-power} of the circuit, as a consequence of Kirchhoff's laws (Lemma \ref{lma:power} in Appendix \ref{sec:proof}). By bridging the conceptual gap between energy-based models (at a mathematical level\footnote{In an EBM, the \textit{energy function} is a mathematical abstraction of the model, not a physical energy.}), and physical energies\footnote{Specifically the power dissipated in resistive devices.} (at a hardware level), our work thus introduces an implementation of energy-based neural networks grounded in device physics. 
\item We show how these analog neural networks can be trained with the Equilibrium Propagation framework \citep{Scellier+Bengio-frontiers2017}, and we prove a formula for updating the conductances (the synaptic weights) in proportion to their error gradients, using solely the voltage drops across the corresponding resistive devices (Theorem \ref{thm:gradients} in Section \ref{sec:eqprop}).
This result provides theoretical ground for implementing end-to-end analog neural networks trainable by stochastic gradient descent using a fast and low-power weight-update mechanism. 
\item We introduce a deep analog network architecture inspired by those of conventional deep learning (Fig.\ref{fig:network} and Section \ref{sec:model}).
\item We demonstrate the potential of our novel neuromorphic hardware methodology with numerical simulations on the MNIST dataset, using a SPICE-based framework to simulate the circuit's dynamics. Due to computational constraints, we train a network with 100 hidden neurons for 10 epochs and obtain $3.43\%$ test error rate, outperforming equivalent-size software-based neural networks (Section \ref{sec:numerical-simulations}).
\end{itemize}

By explicitly decoupling the training algorithm (EqProp in Section \ref{sec:eqprop}) from the specific neural network architecture studied in this work (Section \ref{sec:model}), we stress that our optimization method can be used for any network architecture, not just the one of Section \ref{sec:model}. Our modular approach thus offers the possibility to explore the design space of analog network architectures trainable with EqProp, in essentially the same way as deep learning researchers explore the design space of differentiable neural networks trainable with backpropagation.

\begin{figure*}[ht!]
\begin{center}
\includegraphics[width=\textwidth]{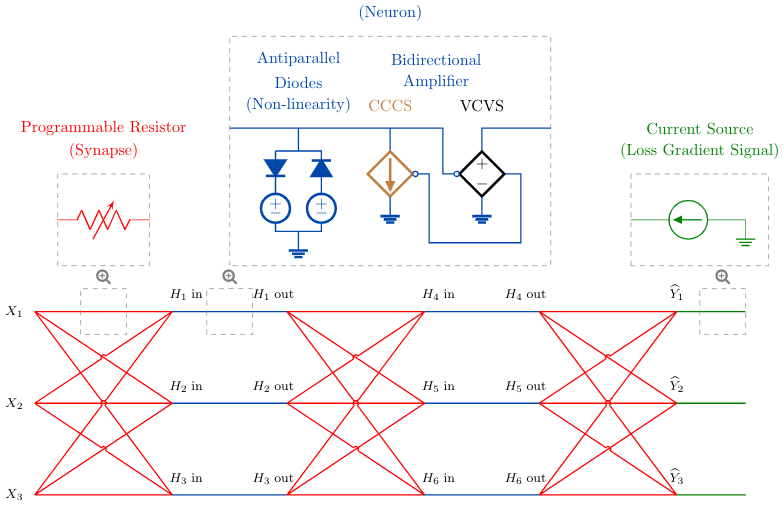}
\end{center}
\caption{
\textbf{Deep analog neural network} with three input nodes ($X_1$, $X_2$ and $X_3$), two layers of three hidden neurons each ($H_1$, $H_2$, $H_3$, and $H_4$, $H_5$, $H_6$) and three output nodes ($\widehat{Y}_1$, $\widehat{Y}_2$ and $\widehat{Y}_3$). Blue branches and red branches represent neurons and synapses, respectively. Each synapse is a programmable resistor, whose conductance represents a parameter to be adjusted (section \ref{sec:synapses}). Each neuron is formed of a nonlinear transfer function and a bidirectional amplifier. The nonlinear transfer function is implemented by a pair of antiparallel diodes (in series with voltage sources), which forms a sigmoidal function in its voltage response (section \ref{sec:nonlinear-transfer-function}). The bidirectional amplifier consists of a current-controlled current source (CCCS, shown in brown) and a voltage-controlled voltage source (VCVS, shown in black), allowing signals to propagate in both directions without a decay in amplitude (section \ref{sec:amplifiers}). Output nodes are linked to current sources (shown in green) which serve to inject loss gradient signals during training (section \ref{sec:switches}). \textbf{Equilibrium Propagation} is a two-phase procedure to compute the gradient of a loss $\mathcal{L} = \ell(\widehat{Y},Y)$, where $Y$ is the desired target (section \ref{sec:conductance-gradient}). In the first phase (\textit{free phase}, or inference), input voltages are sourced at input nodes and the current sources are set to zero; in the second phase (\textit{nudged phase}), for each output node $\widehat{Y}_k$ the corresponding current source is set to $I_k = - \beta \frac{\partial \ell}{\partial \widehat{Y}_k}$, where $\beta$ is a scaling factor (a hyperparameter). The update rule to adjust the conductances of programmable resistors is local (Theorem \ref{thm:gradients}.)
}
\label{fig:network}
\end{figure*}

\section{Related Work}

\paragraph{Accelerators for deep learning.}
The computations involved in conventional neural networks consist in large parts of tensor multiplications. A key insight is that the multiply and accumulate operations of tensor multiplications can be performed in the analog domain by utilizing Ohm's law and Kirchhoff's current law, respectively. In particular, crossbar arrays of analog resistive memory elements can implement matrix-vector multiplications in their voltage-current transfer function. A growing field of research exploits this idea to develop specialized hardware aimed at speeding up the forward and backward passes of the backpropagation algorithm \citep{burr2017neuromorphic,jerry2017ferroelectric,ambrogio2018equivalent,li2018efficient,xia2019memristive}. However, in this approach, a fundamental drawback arises from utilizing a different circuit for the forward and backward passes. The device mismatches and nonidealities inevitably present in analog hardware cause a layer-specific error in gradient computation. As the gradients propagate backwards through the network, these errors accumulate. As a result, the performance of the network is severely degraded. 
This effect can be mitigated by using digital-to-analog conversion (DAC) and analog-to-digital conversion (ADC) between the layers of the network to compute the forward activation function and the backward pointwise derivatives of the activation function in the digital domain (`exactly'), but this technique comes at the cost of greatly increasing power consumption \citep{li2015merging}. In contrast, our method (EqProp) uses the same circuit for both phases of training (free phase and nudged phase), thereby greatly simplifying the resulting hardware architecture. Crucially, since our update rule is local (Theorem \ref{thm:gradients}), the nonidealities of devices do not prevent reliable gradient computation. Finally, our method does not require DACs and ADCs.

\paragraph{Energy-based models (EBMs).}
EBMs have played an important role in the foundations of deep learning -- see \citet{lecun2006tutorial} for a comprehensive tutorial on energy-based learning.
Historically, the Hopfield model (first introduced in the discrete time setting with binary neurons \citep{hopfield1982neural} and then adapted to the real-time setting with real-valued neurons \citep{cohen1983absolute,hopfield1984neurons}) and the Boltzmann machine \citep{ackley1985learning} (a stochastic variant of the former) were the first EBMs to be introduced. Unfortunately, these models suffer from long and slow inference phases, making them mostly impractical and out of stage in modern deep learning applications. Boltzmann machines require running a Monte Carlo Markov Chain (MCMC), while the Hopfield model requires an equally long phase of energy minimization.
By highlighting the fact that nonlinear resistive networks possess an energy function (the so-called \textit{total pseudo-power}), our work suggests an ultra efficient implementation of energy-based neural networks, in which minimization of the energy function is performed by the physics of the circuit (Kirchhoff's laws), rather than by a lengthy MCMC simulation.

\paragraph{Equilibrium Propagation (EqProp).}
EqProp is a general algorithm for computing error gradients in EBMs, inspired by the contrastive learning algorithm for Boltzmann machines \citep{ackley1985learning} and Hopfield networks \citep{movellan1991contrastive}. Previous works have mostly studied EqProp in the setting of classification tasks to train the Hopfield model and variants of it \citep{Scellier+Bengio-frontiers2017,scellier2019equivalence,scellier2018generalization,khan2018bidirectional,o2018initialized,o2019training,ernoult2019updates,ernoult2020equilibrium,zoppo2020equilibrium}. However, EqProp is a much more general method. Recently, \citet{ernoult2019updates} used EqProp to train energy-based convolutional networks. In this work, we apply EqProp to a new class of energy-based neural networks called \textit{nonlinear resistive networks} and we show in Appendix \ref{sec:general-applicability} how EqProp can be used in the setting of Generative Adversarial Networks \citep{goodfellow2014generative}.

\section{End-to-End Training via Equilibrium Propagation}
\label{sec:eqprop}

We consider here for simplicity of presentation the supervised setting in which one has an input $X$ and one wants to predict its associated target $Y$, e.g. the setting of image classification where $X$ is an image and $Y$ a label. Note however that our framework extends beyond this setting: we show in Appendix \ref{sec:general-applicability} how it can be used in the setting of generative adversarial learning.

\subsection{Nonlinear Resistive Networks and Supervised Learning}
\label{sec:electrical-networks}

\paragraph{Nonlinear resistive network.} The neural networks studied here are called \textit{nonlinear resistive networks}. These are electrical circuits consisting of two-terminal elements with continuous current-voltage characteristics. This includes programmable resistors (whose conductances play the role of synaptic weights), diodes (which form non-linear transfer functions), voltage sources (used to set data samples at input nodes) and current sources (used to inject loss gradient signals during training).

\paragraph{Performing inference.} A subset of the nodes of the circuit are \textit{input nodes} at which input voltages (denoted $X$) are sourced. All other nodes -- the \textit{internal nodes} and \textit{output nodes} -- are left floating: after the voltages of input nodes have been set, the voltages of internal and output nodes settle to their steady state. The output nodes, denoted $\widehat{Y}$, represent the readout of the system, i.e. the model prediction.

\paragraph{Loss function to minimize.}
The architecture and the components of the circuit determine the $X \mapsto \widehat{Y}$ function. Specifically, the conductances of the programmable resistors, denoted $\theta$, parameterize this function. That is, $\widehat{Y}$ can be written as a function of $X$ and $\theta$ in the form $\widehat{Y}(X,\theta)$. Training such a circuit consists in adjusting the values of the conductances ($\theta$) so that the voltages of output nodes ($\widehat{Y}$) approach the target voltages ($Y$). Formally, we cast the goal of training as an optimization problem in which the loss to be optimized (corresponding to an input-target pair $(X,Y)$) is of the form:
\begin{equation}
    \label{eq:loss}
    \mathcal{L}(X,Y,\theta) = \ell \left( \widehat{Y}(X,\theta),Y \right).
\end{equation}
In this work we use the squared error loss (Eq.~\ref{eq:loss3}). However, our framework applies to any differentiable function $\ell$ ; see Appendix \ref{sec:loss} for common examples.

\subsection{Computing the Gradient with respect to the Conductance of a Resistor}
\label{sec:conductance-gradient}

Equilibrium Propagation (EqProp) is a training framework for energy-based models (EBMs) \citep{Scellier+Bengio-frontiers2017}. A key point we show is that nonlinear resistive networks are EBMs\footnote{Writing the form of the `energy function' (the \textit{total pseudo-power} of the circuit) requires introducing a substantial amount of notation. For this reason, we state and prove this result in Appendix \ref{sec:proof}.} (Lemma \ref{lma:power} in Appendix \ref{sec:general-theorem}). This allows us to use EqProp in such analog neural networks to compute the gradient of the loss of Eq.~\ref{eq:loss}. Theorem \ref{thm:gradients} below provides a formula for computing the loss gradient with respect to a conductance using solely the voltage drop across the corresponding resistor.

Given an input $X$ and associated target $Y$, EqProp proceeds in the following two phases.

\paragraph{Free phase (first phase).}
At inference, input voltages are sourced at input nodes ($X$), while all other nodes of the circuit (the internal nodes and output nodes) are left floating. All internal and output node voltages are measured. In particular, the voltages of output nodes ($\widehat{Y}$) corresponding to prediction are compared with the target ($Y$) to compute the loss $\mathcal {L} = \ell ( \widehat{Y},Y)$.

\paragraph{Nudged phase (second phase).}
For each output node $\widehat{Y}_k$, a current $I_k = - \beta \frac{\partial \ell}{\partial \widehat{Y}_k}$ is sourced at $\widehat{Y}_k$, where $\beta$ is a positive or negative scaling factor (a hyperparameter). All internal node voltages and output node voltages are measured anew. \\

\begin{restatable}[Gradient Formula]{thm}{thmgradients}
\label{thm:gradients}
Consider a nonlinear resistive network, and let $g_{ij}$ denote the conductance of a linear resistor whose terminals are $i$ and $j$ (i.e. a resistor across which the current $I_{ij}$ and voltage drop $\Delta V_{ij}$ satisfy Ohm's law: $I_{ij} = g_{ij} \Delta V_{ij}$). Denote $\Delta V_{ij}^0$ the voltage drop across this resistor in the free phase (when no current is sourced at output nodes), and $\Delta V_{ij}^\beta$ the voltage drop in the nudged phase (when a current $I_k = - \beta \frac{\partial \ell}{\partial \widehat{Y}_k}$ is sourced at each output node $\widehat{Y}_k$). Then, the gradient of the loss ${\mathcal L}  = \ell \left( \widehat{Y},Y \right)$ with respect to $g_{ij}$ can be estimated, in the limit $\beta \to 0$, as
\begin{equation}
\label{eq:thm-conductance}
\frac{\partial {\mathcal L}}{\partial g_{ij}} = \lim_{\beta \to 0} \frac{1}{2 \beta} \left( \left( \Delta V^\beta_{ij} \right)^2 -  \left( \Delta V^0_{ij} \right)^2 \right).
\end{equation}
\end{restatable}

Theorem \ref{thm:gradients} is proved in Appendix \ref{sec:proof} (the proof comes with a detailed sketch of the proof first). It is a particular case of the more general formula of Theorem \ref{lma:gradients} (Appendix \ref{sec:parameter-gradient}) which shows how to compute the gradient in the case of an arbitrary non-linear resistive device.

In Appendix \ref{sec:loss}, we give a few examples of common loss functions and derive the corresponding currents ($I_k$) to be sourced at output nodes in the nudged phase.

Interestingly, in a chain-like layered neural network such as the one of Figure \ref{fig:network}, the second phase of EqProp (nudged phase) is similar in spirit to the backward pass of the backpropagation algorithm: the currents introduced at output nodes in the nudged phase can be thought of as error signals propagating backwards in the layers of the neural network, from output nodes back to input nodes.

\section{Deep Analog Neural Network Architecture}
\label{sec:model}

The theory of Section \ref{sec:eqprop} applies to any nonlinear resistive network. In this section, we introduce a neural network architecture inspired by those of conventional deep learning (Figure~\ref{fig:network}). It is composed of multiple layers, alternating linear and non-linear processing stages. The linear transformations are performed by crossbar arrays of programmable resistors, whose conductances play the role of synaptic weights that parameterize the transformations. The nonlinear transfer function is implemented using a pair of diodes, followed by a linear amplifier. These crossbar arrays of programmable resistors and these nonlinear transfer functions are alternated to form a deep network.

\subsection{Programmable Resistors as Synapses}
\label{sec:synapses}

In the last decade, important advances in nanotechnology have provided neuromorphic researchers with a panoply of new devices
which allow for ultra low-power synaptic plasticity. Programmable resistors that have been proposed and tested as synapses include memristors \citep{jo2010nanoscale}, resistive random-access memory \citep{huang2019hardware}, phase-change memory \citep{ambrogio2016unsupervised}, ferroelectric field-effect transistors \citep{jerry2017ferroelectric
}, flash memory \citep{guo2017fast}, magnetic random-access memory \citep{patil2019mram}, conductive-bridging random access memory \citep{cha2020conductive} and spin-transfer-torque memory \citep{vincent2015spin}, among others. We refer to \citet{burr2017neuromorphic} 
for a review on the use of programmable resistors for neuromorphic computing.

These programmable resistors behave as pure resistors in a low-voltage regime. However, by applying voltage or current pulses, depending on the device, their conductances can be modified. The programming of a conductance requires only a small amount of energy, since such devices can be extremely small ($\sim 2 \; \textrm{nm}$) and their switching speed can be extremely fast ($\sim 1\;\textrm{ns}$).

Programmable resistors are especially suited for a local learning rule such as the one of Theorem \ref{thm:gradients}. In the free phase and nudged phase, the node voltages can be measured (without disturbing the circuit) and stored by using a sample-and-hold amplifier (SHA)
. Then, by embedding the resistors in a suitable synaptic circuit, we can program them to update their conductances proportionally to their loss gradients, i.e. $\Delta g_{ij} \propto - \frac{1}{\beta} \left[ \left( \Delta V^\beta_{ij} \right)^2 -  \left( \Delta V^0_{ij} \right)^2 \right]$, thus performing one step of stochastic gradient descent (SGD). This can be achieved by means of amplitude/duration modulation of a voltage or current pulse applied to the resistive device, for example the 1T–1R (one transistor–one memristor) synapse \citep{merced2016repeatable}. Furthermore, in a crossbar array, all conductances can be updated simultaneously with two vectors of voltage pulses, a method called the \textit{outer product update} \citep{fuller2019parallel}.

Although the physical realization of memristor synapses presents challenges (such as device-to-device variability and systematic bias in the weight updates), its investigation is out of the scope of this work. We refer to \citet{chang2017mitigating} for a thorough analysis.

\subsection{Neurons as Nonlinear Transfer Functions}
\label{sec:nonlinear-transfer-function}

One way to think of a neuron is that it acts as a transfer function between an input current 
and an output voltage (Fig.~\ref{fig:network}). The transfer function (which plays the role of `activation function') is set up so that the neuron's output voltage is a smooth and S-shaped (sigmoidal) function of the neuron's input current. 
To achieve this, we place two diodes antiparallel between the neuron's input node and ground, in series with voltage sources (used to shift the bounds of the activation function). As a result, the neuron’s output voltage is linear in the range between zero and one volt, but as the voltage rises above or below these thresholds, one of the diodes turns on and sinks most of the extra current to ground. The output voltage remains bounded even as the input current 
grows very large. The transfer function equation of a diode is provided in Appendix \ref{sec:model-details}.

\subsection{Bidirectional Amplifiers}
\label{sec:amplifiers}

In a network consisting only of resistors and diodes, simulations show that the voltages of hidden neurons span a significantly smaller range of voltage values than the voltages of input nodes (the voltages of hidden neurons are closer to zero). To prevent signal decay, we use voltage-controlled voltage sources (VCVS) which amplify the voltages of hidden neurons in the forward direction by a gain factor $A$. To better propagate error signals in the second phase of training (nudged phase), we also use current-controlled current sources (CCCS) which amplify currents in the backward direction by a gain factor $1/A$. We call such a combination of a forward-directed VCVS and a backward-directed CCCS a `bidirectional amplifier'. More details are provided in Appendix \ref{sec:amplifiers-details}.

\subsection{Constraint of Positive Weights -- Doubling Input and Output Nodes}
\label{sec:pos-weights}

One constraint of analog neural networks (compared to conventional neural networks) is that the conductances of programmable resistors, which represent the weights, are positive. Several approaches are proposed in the literature to overcome this constraint (Appendix \ref{sec:nonnegative-weights}). In this work, our approach consists in doubling the number of input nodes and inverting one set, and doubling the number of output nodes.
Typically in a classification task with $K$ classes, the target vector $Y = (Y_1,Y_2, \ldots, Y_K)$ represents the one-hot code of the class label.
For each class $k$, our network thus has two output nodes $\widehat{Y}_k^+$ and $\widehat{Y}_k^-$, with $\widehat{Y}_k^+ - \widehat{Y}_k^-$ representing a score assigned to class $k$. The loss to optimize is the squared error loss:
\begin{equation}
\label{eq:loss3}
\ell(\widehat{Y},Y) = \frac{1}{2} \sum_{k=1}^K \left( \widehat{Y}_k^+-\widehat{Y}_k^--Y_k \right)^2.
\end{equation}

\subsection{Injecting Loss Gradient Signals with Current Sources}
\label{sec:switches}

In the nudged phase (second phase), we require currents $I_k^+$ and $I_k^-$ proportional to the gradients of output node voltages $\widehat{Y}_k^+$ and $\widehat{Y}_k^-$, which must be injected at output nodes. These currents are:
\begin{equation}
    I_k^+ = - \beta \frac{\partial \ell}{\partial \widehat{Y}_k^+} = \beta \left( Y_k+\widehat{Y}_k^--\widehat{Y}_k^+ \right), \qquad I_k^- = - \beta \frac{\partial \ell}{\partial \widehat{Y}_k^-} = \beta \left( \widehat{Y}_k^+-\widehat{Y}_k^- - Y_k \right).
\end{equation}
We achieve this using current sources. In the free phase (first phase), the same current sources are set to zero current, acting like open circuits and not influencing the voltages of output nodes.

\section{Numerical Simulations}
\label{sec:numerical-simulations}

We use the SPICE (simulation program with integrated circuit emphasis) framework for realistic simulations of the circuit’s dynamics \citep{vogt2020ngspice}. See Appendix \ref{sec:simulation-details} for simulation details.

\paragraph{XOR task.}
As a proof of concept, we show that our analog neural network can indeed learn a nonlinear function such as $Y = X_1 \; \textrm{XOR} \; X_2$. The architecture together with the final weights after training are reported in Figure \ref{fig:xor} (Appendix \ref{sec:simulation-details}).

\paragraph{MNIST digits classification task.}
We perform our simulations on the MNIST dataset using the high-performance SPICE-class parallel circuit simulator \textit{Spectre} \citep{2020spectre}. Despite the use of this high-performance simulator, the computational difficulties described below constrain us to limit our simulations to training a small (by deep learning standards) network with a single hidden layer of 100 neurons. Since training takes $18$ hours per epoch, we stop training after 10 epochs (for a total training duration of one week).

After 10 epochs of training, our SPICE-based network achieves a test error rate of $3.43\%$ (Table \ref{mnist_results}), while the training curve (Figure~\ref{fig:training-curves} in Appendix \ref{sec:simulation-details}) suggests that training isn't complete. For comparison, we train a logistic regression model which achieves $7.27\%$ test error (Appendix \ref{sec:simulation-details}), thus demonstrating that our SPICE-based network benefits from the non-linearities offered by the diodes. We then benchmark our SPICE-based network against a PyTorch implementation of the original EqProp model \citep{Scellier+Bengio-frontiers2017}. We train two kinds of PyTorch-based networks: one whose weights are free to be either positive or negative, and one that also possesses the same strictly positive weight constraints as our SPICE-based network (see Section \ref{sec:pos-weights}). In both cases, our SPICE-based network outperforms the PyTorch-based networks with 100 hidden neurons (Table \ref{mnist_results}). We also show that increasing the number of hidden neurons in the PyTorch-based networks results in lower, more competitive test error rates (2.01\%). These results suggest that with more computational power, training a larger SPICE-based network to convergence would result in a lower final test error rate. Altogether our results demonstrate the potential of our novel neuromorphic approach.

\paragraph{Circuit simulators and machine learning.}
We note that SPICE (and similar circuit simulators as well) is not ideal for the types of simulations required in machine learning, due to the sequential nature of training.
Traditionally, one designs a circuit and uses a simulator to verify the correctness of it via a tractable number of test cases. Importantly, these simulators are not designed for performing many simulations iteratively.
Perhaps interestingly, we found that simulators often possessed a large overhead that surpassed the actual simulation time, 
which further points out that these simulators are not designed for running millions of iterations.

\begin{table}
    \caption{Results on MNIST. SPICE EqProp (our model) is benchmarked against a PyTorch implementation of the original EqProp model \citep{Scellier+Bengio-frontiers2017}. For the PyTorch models, the mean values and 95\% confidence intervals of test errors are reported over five runs, after 10 epochs of training and after training completion. 100 and 500 denote the number of neurons in the hidden layer. `pos. weights' means that weights are constrained to be positive (as explained in section \ref{sec:pos-weights}). }
\centering
\begin{tabular}{@{}lccccc@{}} \toprule
{}&\multicolumn{2}{c}{Error rates after $10$ epochs ($\%$)} &\multicolumn{2}{c}{Final error rates ($\%$)} \\
\cmidrule(r){2-5}
{}& Test & (Train) & Test & (Train) \\
\midrule
\textbf{SPICE EqProp 100 (our model)} & $\mathbf{3.43}$ & $(\mathbf{2.68})$ & & \\ 
\midrule
PyTorch EqProp 100 (pos. weights) & $3.85 \pm 0.05$  & $(2.87)$ & $3.44 \pm 0.04$ & $(1.38)$ \\
PyTorch EqProp 100                & $3.99 \pm 0.19$  & $(2.93)$ & $3.59 \pm 0.06$ & $(1.35)$ \\
\midrule
PyTorch EqProp 500 (pos. weights) & $2.49 \pm 0.01$  & $(1.47)$ & $\mathbf{2.01 \pm 0.02}$ & $(\mathbf{0.26})$ \\
PyTorch EqProp 500                & $2.92 \pm 0.09$  & $(1.82)$ & $2.38 \pm 0.11$ & $(0.52)$ \\
\bottomrule
\end{tabular}
\label{mnist_results}
\end{table}

\section{Discussion}

Even though the idea of neuromorphic hardware has existed for long \citep{mead1989analog}, in the last decades neural networks have more rapidly evolved to fit the constraints of conventional von Neumann computers. As a consequence, today's neural networks (together with the backpropagation algorithm which was invented to train them) are fundamentally discrete-time dynamical systems. With the recent success of deep learning, a growing field of research is emerging which uses mixed signal analog/digital hardware to accelerate these discrete-time neural networks. However, when moving from the digital world to the analog world, much more improvement in speed and power efficiency is possible by rethinking neural networks as well as their training algorithm. Alternative neural network paradigms which make use of the full potential of analog hardware for ultra-low energy computing have been proposed and studied (e.g. spiking neural networks) but their development has thus far been hindered due to the lack of a theoretical framework to train them.

Our work uses the formalism of energy-based models and equilibrium propagation to provide such a theoretical framework for end-to-end analog neural networks trainable by stochastic gradient descent with a local weight update mechanism. The modularity of our framework, in which the training algorithm (EqProp) can be decoupled from the neural network architecture, offers the possibility for neuromorphic researchers to explore the design space of analog network architectures, with the perspective of finding circuits and configurations of components that best fit with EqProp training.

Among the remaining challenges, the main one is perhaps to figure out how to leverage the working mechanisms of memristors to implement reliable conductance changes. While experimental implementation of memristive crossbar arrays is still in its infancy, numerous simulations have indicated their potential. In particular, \citet{ambrogio2018equivalent} have demonstrated in the context of feedforward networks trained with backpropagation that memristive crossbar arrays can achieve comparable accuracy to software-based networks running on conventional computer hardware, while increasing speed and reducing power consumption by two orders of magnitude.
The potential speedup and power reduction offered by analog computing
are especially critical for scaling up neural networks to billions of neurons (sizes that are far out of reach with current GPU-based deep learning models). To achieve such sizes, it will also become critical to take into account geometric constraints such as sparsity in neural connectivity. For this purpose, sparse analog hardware architectures such as Memristive Nanowire Neural Networks (MN3) are a promising alternative to dense crossbar arrays \citep{kendall2020deep}.

Finally, our framework, which is a theory of how circuit dynamics can optimize objective functions, may also inspire neuroscientists who seek to explain the mechanisms of credit assignment in the brain \citep{whittington2019theories,richards2019deep,lillicrap2020backpropagation}.

\section*{Acknowledgments}

The authors would like to thank Thomas Fischbacher, Maxence Ernoult, Michal Januszewski and Effrosyni Kokiopoulou for valuable feedback and discussions. The authors would also like to thank John O'Donovan for his invaluable assistance with the Spectre simulator.

\bibliographystyle{abbrvnat}
\bibliography{biblio}

\begin{thebibliography}{53}
\providecommand{\natexlab}[1]{#1}
\providecommand{\url}[1]{\texttt{#1}}
\expandafter\ifx\csname urlstyle\endcsname\relax
  \providecommand{\doi}[1]{doi: #1}\else
  \providecommand{\doi}{doi: \begingroup \urlstyle{rm}\Url}\fi

\bibitem[Ackley et~al.(1985)Ackley, Hinton, and Sejnowski]{ackley1985learning}
D.~H. Ackley, G.~E. Hinton, and T.~J. Sejnowski.
\newblock A learning algorithm for boltzmann machines.
\newblock \emph{Cognitive science}, 9\penalty0 (1):\penalty0 147--169, 1985.

\bibitem[Ambrogio et~al.(2016)Ambrogio, Ciocchini, Laudato, Milo, Pirovano,
  Fantini, and Ielmini]{ambrogio2016unsupervised}
S.~Ambrogio, N.~Ciocchini, M.~Laudato, V.~Milo, A.~Pirovano, P.~Fantini, and
  D.~Ielmini.
\newblock Unsupervised learning by spike timing dependent plasticity in phase
  change memory (pcm) synapses.
\newblock \emph{Frontiers in neuroscience}, 10:\penalty0 56, 2016.

\bibitem[Ambrogio et~al.(2018)Ambrogio, Narayanan, Tsai, Shelby, Boybat,
  di~Nolfo, Sidler, Giordano, Bodini, Farinha, et~al.]{ambrogio2018equivalent}
S.~Ambrogio, P.~Narayanan, H.~Tsai, R.~M. Shelby, I.~Boybat, C.~di~Nolfo,
  S.~Sidler, M.~Giordano, M.~Bodini, N.~C. Farinha, et~al.
\newblock Equivalent-accuracy accelerated neural-network training using
  analogue memory.
\newblock \emph{Nature}, 558\penalty0 (7708):\penalty0 60--67, 2018.

\bibitem[Baez and Fong(2015)]{baez2015compositional}
J.~C. Baez and B.~Fong.
\newblock A compositional framework for passive linear networks.
\newblock \emph{arXiv preprint arXiv:1504.05625}, 2015.

\bibitem[Burr et~al.(2017)Burr, Shelby, Sebastian, Kim, Kim, Sidler, Virwani,
  Ishii, Narayanan, Fumarola, et~al.]{burr2017neuromorphic}
G.~W. Burr, R.~M. Shelby, A.~Sebastian, S.~Kim, S.~Kim, S.~Sidler, K.~Virwani,
  M.~Ishii, P.~Narayanan, A.~Fumarola, et~al.
\newblock Neuromorphic computing using non-volatile memory.
\newblock \emph{Advances in Physics: X}, 2\penalty0 (1):\penalty0 89--124,
  2017.

\bibitem[{Cadence Design Systems, Inc.}(2020)]{2020spectre}
{Cadence Design Systems, Inc.}
\newblock Spectre circuit simulator reference, version 19.1, Jan 2020.

\bibitem[Cha et~al.(2020)Cha, Yang, Oh, Choi, Park, Jang, Ahn, and
  Choi]{cha2020conductive}
J.-H. Cha, S.~Y. Yang, J.~Oh, S.~Choi, S.~Park, B.~C. Jang, W.~B. Ahn, and
  S.-Y. Choi.
\newblock Conductive-bridging random-access memories for emerging neuromorphic
  computing.
\newblock \emph{Nanoscale}, 2020.

\bibitem[Chang et~al.(2017)Chang, Chen, Chou, Wang, Hudec, Chang, Tsai, Chang,
  and Hou]{chang2017mitigating}
C.-C. Chang, P.-C. Chen, T.~Chou, I.-T. Wang, B.~Hudec, C.-C. Chang, C.-M.
  Tsai, T.-S. Chang, and T.-H. Hou.
\newblock Mitigating asymmetric nonlinear weight update effects in hardware
  neural network based on analog resistive synapse.
\newblock \emph{IEEE Journal on Emerging and Selected Topics in Circuits and
  Systems}, 8\penalty0 (1):\penalty0 116--124, 2017.

\bibitem[Christianson and Erickson(2007)]{christianson2007dirichlet}
K.~Christianson and L.~Erickson.
\newblock The dirichlet problem on directed networks.
\newblock 2007.
\newblock URL
  \url{https://sites.math.washington.edu/~reu/papers/2007/KariLindsay/dirichlet.pdf}.
\newblock Accessed: 2020-05-31.

\bibitem[Chua(1971)]{chua1971memristor}
L.~Chua.
\newblock Memristor-the missing circuit element.
\newblock \emph{IEEE Transactions on circuit theory}, 18\penalty0 (5):\penalty0
  507--519, 1971.

\bibitem[Cohen and Grossberg(1983)]{cohen1983absolute}
M.~A. Cohen and S.~Grossberg.
\newblock Absolute stability of global pattern formation and parallel memory
  storage by competitive neural networks.
\newblock \emph{IEEE transactions on systems, man, and cybernetics}, \penalty0
  (5):\penalty0 815--826, 1983.

\bibitem[Ernoult et~al.(2019)Ernoult, Grollier, Querlioz, Bengio, and
  Scellier]{ernoult2019updates}
M.~Ernoult, J.~Grollier, D.~Querlioz, Y.~Bengio, and B.~Scellier.
\newblock Updates of equilibrium prop match gradients of backprop through time
  in an rnn with static input.
\newblock In \emph{Advances in Neural Information Processing Systems}, pages
  7079--7089, 2019.

\bibitem[Ernoult et~al.(2020)Ernoult, Grollier, Querlioz, Bengio, and
  Scellier]{ernoult2020equilibrium}
M.~Ernoult, J.~Grollier, D.~Querlioz, Y.~Bengio, and B.~Scellier.
\newblock Equilibrium propagation with continual weight updates.
\newblock \emph{arXiv preprint arXiv:2005.04168}, 2020.

\bibitem[Fuller et~al.(2019)Fuller, Keene, Melianas, Wang, Agarwal, Li,
  Tuchman, James, Marinella, Yang, et~al.]{fuller2019parallel}
E.~J. Fuller, S.~T. Keene, A.~Melianas, Z.~Wang, S.~Agarwal, Y.~Li, Y.~Tuchman,
  C.~D. James, M.~J. Marinella, J.~J. Yang, et~al.
\newblock Parallel programming of an ionic floating-gate memory array for
  scalable neuromorphic computing.
\newblock \emph{Science}, 364\penalty0 (6440):\penalty0 570--574, 2019.

\bibitem[Glorot and Bengio(2010)]{glorot2010understanding}
X.~Glorot and Y.~Bengio.
\newblock Understanding the difficulty of training deep feedforward neural
  networks.
\newblock In \emph{Proceedings of the Thirteenth International Conference on
  Artificial Intelligence and Statistics}, pages 249--256, 2010.

\bibitem[Goodfellow et~al.(2014)Goodfellow, Pouget-Abadie, Mirza, Xu,
  Warde-Farley, Ozair, Courville, and Bengio]{goodfellow2014generative}
I.~Goodfellow, J.~Pouget-Abadie, M.~Mirza, B.~Xu, D.~Warde-Farley, S.~Ozair,
  A.~Courville, and Y.~Bengio.
\newblock Generative adversarial nets.
\newblock In \emph{Advances in neural information processing systems}, pages
  2672--2680, 2014.

\bibitem[Guo et~al.(2017)Guo, Bayat, Bavandpour, Klachko, Mahmoodi, Prezioso,
  Likharev, and Strukov]{guo2017fast}
X.~Guo, F.~M. Bayat, M.~Bavandpour, M.~Klachko, M.~Mahmoodi, M.~Prezioso,
  K.~Likharev, and D.~Strukov.
\newblock Fast, energy-efficient, robust, and reproducible mixed-signal
  neuromorphic classifier based on embedded nor flash memory technology.
\newblock In \emph{2017 IEEE International Electron Devices Meeting (IEDM)},
  pages 6--5. IEEE, 2017.

\bibitem[He et~al.(2016)He, Zhang, Ren, and Sun]{he2016deep}
K.~He, X.~Zhang, S.~Ren, and J.~Sun.
\newblock Deep residual learning for image recognition.
\newblock In \emph{Proceedings of the IEEE conference on computer vision and
  pattern recognition}, pages 770--778, 2016.

\bibitem[Hinton et~al.(2012)Hinton, Deng, Yu, Dahl, Mohamed, Jaitly, Senior,
  Vanhoucke, Nguyen, Sainath, et~al.]{hinton2012deep}
G.~Hinton, L.~Deng, D.~Yu, G.~E. Dahl, A.-r. Mohamed, N.~Jaitly, A.~Senior,
  V.~Vanhoucke, P.~Nguyen, T.~N. Sainath, et~al.
\newblock Deep neural networks for acoustic modeling in speech recognition: The
  shared views of four research groups.
\newblock \emph{IEEE Signal processing magazine}, 29\penalty0 (6):\penalty0
  82--97, 2012.

\bibitem[Hopfield(1982)]{hopfield1982neural}
J.~J. Hopfield.
\newblock Neural networks and physical systems with emergent collective
  computational abilities.
\newblock \emph{Proceedings of the national academy of sciences}, 79\penalty0
  (8):\penalty0 2554--2558, 1982.

\bibitem[Hopfield(1984)]{hopfield1984neurons}
J.~J. Hopfield.
\newblock Neurons with graded response have collective computational properties
  like those of two-state neurons.
\newblock \emph{Proceedings of the national academy of sciences}, 81\penalty0
  (10):\penalty0 3088--3092, 1984.

\bibitem[Hoskins et~al.(2019)Hoskins, Daniels, Huang, Madhavan, Adam, Zhitenev,
  McClelland, and Stiles]{hoskins2019sbe}
B.~D. Hoskins, M.~W. Daniels, S.~Huang, A.~Madhavan, G.~C. Adam, N.~Zhitenev,
  J.~J. McClelland, and M.~D. Stiles.
\newblock Streaming batch eigenupdates for hardware neural networks.
\newblock \emph{Frontiers in Neuroscience}, 13, 2019.

\bibitem[Hu et~al.(2016)Hu, Strachan, Li, Grafals, Davila, Graves, Lam, Ge,
  Yang, and Williams]{hu2016dot}
M.~Hu, J.~P. Strachan, Z.~Li, E.~M. Grafals, N.~Davila, C.~Graves, S.~Lam,
  N.~Ge, J.~J. Yang, and R.~S. Williams.
\newblock Dot-product engine for neuromorphic computing: Programming 1t1m
  crossbar to accelerate matrix-vector multiplication.
\newblock In \emph{2016 53nd ACM/EDAC/IEEE Design Automation Conference (DAC)},
  pages 1--6. IEEE, 2016.

\bibitem[Huang et~al.(2019)Huang, Zhou, Zhang, Xiang, Han, Liu, Liu, and
  Kang]{huang2019hardware}
P.~Huang, Z.~Zhou, Y.~Zhang, Y.~Xiang, R.~Han, L.~Liu, X.~Liu, and J.~Kang.
\newblock Hardware implementation of rram based binarized neural networks.
\newblock \emph{APL Materials}, 7\penalty0 (8):\penalty0 081105, 2019.

\bibitem[Jerry et~al.(2017)Jerry, Chen, Zhang, Sharma, Ni, Yu, and
  Datta]{jerry2017ferroelectric}
M.~Jerry, P.-Y. Chen, J.~Zhang, P.~Sharma, K.~Ni, S.~Yu, and S.~Datta.
\newblock Ferroelectric fet analog synapse for acceleration of deep neural
  network training.
\newblock In \emph{2017 IEEE International Electron Devices Meeting (IEDM)},
  pages 6--2. IEEE, 2017.

\bibitem[Jo et~al.(2010)Jo, Chang, Ebong, Bhadviya, Mazumder, and
  Lu]{jo2010nanoscale}
S.~H. Jo, T.~Chang, I.~Ebong, B.~B. Bhadviya, P.~Mazumder, and W.~Lu.
\newblock Nanoscale memristor device as synapse in neuromorphic systems.
\newblock \emph{Nano letters}, 10\penalty0 (4):\penalty0 1297--1301, 2010.

\bibitem[Johnson(2010)]{johnson2010nonlinear}
W.~Johnson.
\newblock Nonlinear electrical networks, 2010.
\newblock URL
  \url{https://sites.math.washington.edu/~reu/papers/2017/willjohnson/directed-networks.pdf}.
\newblock Accessed: 2020-05-31.

\bibitem[Kendall et~al.(2020)Kendall, Pantone, and Nino]{kendall2020deep}
J.~D. Kendall, R.~D. Pantone, and J.~C. Nino.
\newblock Deep learning in memristive nanowire networks.
\newblock \emph{arXiv preprint arXiv:2003.02642}, 2020.

\bibitem[Khan(2018)]{khan2018bidirectional}
A.~F. Khan.
\newblock Bidirectional learning in recurrent neural networks using equilibrium
  propagation.
\newblock Master's thesis, University of Waterloo, 2018.

\bibitem[LeCun et~al.(1998)LeCun, Bottou, Bengio, and
  Haffner]{lecun1998gradient}
Y.~LeCun, L.~Bottou, Y.~Bengio, and P.~Haffner.
\newblock Gradient-based learning applied to document recognition.
\newblock \emph{Proceedings of the IEEE}, 86\penalty0 (11):\penalty0
  2278--2324, 1998.

\bibitem[LeCun et~al.(2006)LeCun, Chopra, Hadsell, Ranzato, and
  Huang]{lecun2006tutorial}
Y.~LeCun, S.~Chopra, R.~Hadsell, M.~Ranzato, and F.~Huang.
\newblock A tutorial on energy-based learning.
\newblock \emph{Predicting structured data}, 1\penalty0 (0), 2006.

\bibitem[Li et~al.(2015)Li, Xia, Gu, Wang, and Yang]{li2015merging}
B.~Li, L.~Xia, P.~Gu, Y.~Wang, and H.~Yang.
\newblock Merging the interface: Power, area and accuracy co-optimization for
  rram crossbar-based mixed-signal computing system.
\newblock In \emph{Proceedings of the 52nd Annual Design Automation
  Conference}, pages 1--6, 2015.

\bibitem[Li et~al.(2018)Li, Belkin, Li, Yan, Hu, Ge, Jiang, Montgomery, Lin,
  Wang, et~al.]{li2018efficient}
C.~Li, D.~Belkin, Y.~Li, P.~Yan, M.~Hu, N.~Ge, H.~Jiang, E.~Montgomery, P.~Lin,
  Z.~Wang, et~al.
\newblock Efficient and self-adaptive in-situ learning in multilayer memristor
  neural networks.
\newblock \emph{Nature communications}, 9\penalty0 (1):\penalty0 1--8, 2018.

\bibitem[Lillicrap et~al.(2020)Lillicrap, Santoro, Marris, Akerman, and
  Hinton]{lillicrap2020backpropagation}
T.~P. Lillicrap, A.~Santoro, L.~Marris, C.~J. Akerman, and G.~Hinton.
\newblock Backpropagation and the brain.
\newblock \emph{Nature Reviews Neuroscience}, pages 1--12, 2020.

\bibitem[Mead(1989)]{mead1989analog}
C.~Mead.
\newblock Analog vlsi and neutral systems.
\newblock \emph{NASA STI/Recon Technical Report A}, 90, 1989.

\bibitem[Merced-Grafals et~al.(2016)Merced-Grafals, D{\'a}vila, Ge, Williams,
  and Strachan]{merced2016repeatable}
E.~J. Merced-Grafals, N.~D{\'a}vila, N.~Ge, R.~S. Williams, and J.~P. Strachan.
\newblock Repeatable, accurate, and high speed multi-level programming of
  memristor 1t1r arrays for power efficient analog computing applications.
\newblock \emph{Nanotechnology}, 27\penalty0 (36):\penalty0 365202, 2016.

\bibitem[Movellan(1991)]{movellan1991contrastive}
J.~R. Movellan.
\newblock Contrastive hebbian learning in the continuous hopfield model.
\newblock In \emph{Connectionist models}, pages 10--17. Elsevier, 1991.

\bibitem[O'Connor et~al.(2018)O'Connor, Gavves, and Welling]{o2018initialized}
P.~O'Connor, E.~Gavves, and M.~Welling.
\newblock Initialized equilibrium propagation for backprop-free training.
\newblock 2018.

\bibitem[Oord et~al.(2016)Oord, Dieleman, Zen, Simonyan, Vinyals, Graves,
  Kalchbrenner, Senior, and Kavukcuoglu]{oord2016wavenet}
A.~v.~d. Oord, S.~Dieleman, H.~Zen, K.~Simonyan, O.~Vinyals, A.~Graves,
  N.~Kalchbrenner, A.~Senior, and K.~Kavukcuoglu.
\newblock Wavenet: A generative model for raw audio.
\newblock \emph{arXiv preprint arXiv:1609.03499}, 2016.

\bibitem[O’Connor et~al.(2019)O’Connor, Gavves, and Welling]{o2019training}
P.~O’Connor, E.~Gavves, and M.~Welling.
\newblock Training a spiking neural network with equilibrium propagation.
\newblock In \emph{The 22nd International Conference on Artificial Intelligence
  and Statistics}, pages 1516--1523, 2019.

\bibitem[{Patil} et~al.(2019){Patil}, {Hua}, {Gonugondla}, {Kang}, and
  {Shanbhag}]{patil2019mram}
A.~D. {Patil}, H.~{Hua}, S.~{Gonugondla}, M.~{Kang}, and N.~R. {Shanbhag}.
\newblock An mram-based deep in-memory architecture for deep neural networks.
\newblock In \emph{2019 IEEE International Symposium on Circuits and Systems
  (ISCAS)}, pages 1--5, 2019.

\bibitem[Rekhi et~al.(2019)Rekhi, Zimmer, Nedovic, Liu, Venkatesan, Wang,
  Khailany, Dally, and Gray]{rekhi2019ams}
A.~S. Rekhi, B.~Zimmer, N.~Nedovic, N.~Liu, R.~Venkatesan, M.~Wang,
  B.~Khailany, W.~J. Dally, and C.~T. Gray.
\newblock Analog/mixed-signal hardware error modeling for deep learning
  inference.
\newblock In \emph{Proceedings of the 56th Annual Design Automation Conference
  2019}, DAC ’19, New York, NY, USA, 2019. Association for Computing
  Machinery.
\newblock ISBN 9781450367257.
\newblock \doi{10.1145/3316781.3317770}.
\newblock URL \url{https://doi.org/10.1145/3316781.3317770}.

\bibitem[Richards et~al.(2019)Richards, Lillicrap, Beaudoin, Bengio, Bogacz,
  Christensen, Clopath, Costa, de~Berker, Ganguli, et~al.]{richards2019deep}
B.~A. Richards, T.~P. Lillicrap, P.~Beaudoin, Y.~Bengio, R.~Bogacz,
  A.~Christensen, C.~Clopath, R.~P. Costa, A.~de~Berker, S.~Ganguli, et~al.
\newblock A deep learning framework for neuroscience.
\newblock \emph{Nature neuroscience}, 22\penalty0 (11):\penalty0 1761--1770,
  2019.

\bibitem[Scellier and Bengio(2017)]{Scellier+Bengio-frontiers2017}
B.~Scellier and Y.~Bengio.
\newblock Equilibrium propagation: Bridging the gap between energy-based models
  and backpropagation.
\newblock \emph{Frontiers in computational neuroscience}, 11, 2017.

\bibitem[Scellier and Bengio(2019)]{scellier2019equivalence}
B.~Scellier and Y.~Bengio.
\newblock Equivalence of equilibrium propagation and recurrent backpropagation.
\newblock \emph{Neural computation}, 31\penalty0 (2):\penalty0 312--329, 2019.

\bibitem[Scellier et~al.(2018)Scellier, Goyal, Binas, Mesnard, and
  Bengio]{scellier2018generalization}
B.~Scellier, A.~Goyal, J.~Binas, T.~Mesnard, and Y.~Bengio.
\newblock Generalization of equilibrium propagation to vector field dynamics.
\newblock \emph{arXiv preprint arXiv:1808.04873}, 2018.

\bibitem[Vaswani et~al.(2017)Vaswani, Shazeer, Parmar, Uszkoreit, Jones, Gomez,
  Kaiser, and Polosukhin]{vaswani2017attention}
A.~Vaswani, N.~Shazeer, N.~Parmar, J.~Uszkoreit, L.~Jones, A.~N. Gomez,
  {\L}.~Kaiser, and I.~Polosukhin.
\newblock Attention is all you need.
\newblock In \emph{Advances in neural information processing systems}, pages
  5998--6008, 2017.

\bibitem[{Vincent} et~al.(2015){Vincent}, {Larroque}, {Locatelli}, {Ben
  Romdhane}, {Bichler}, {Gamrat}, {Zhao}, {Klein}, {Galdin-Retailleau}, and
  {Querlioz}]{vincent2015spin}
A.~F. {Vincent}, J.~{Larroque}, N.~{Locatelli}, N.~{Ben Romdhane},
  O.~{Bichler}, C.~{Gamrat}, W.~S. {Zhao}, J.~{Klein}, S.~{Galdin-Retailleau},
  and D.~{Querlioz}.
\newblock Spin-transfer torque magnetic memory as a stochastic memristive
  synapse for neuromorphic systems.
\newblock \emph{IEEE Transactions on Biomedical Circuits and Systems},
  9\penalty0 (2):\penalty0 166--174, 2015.

\bibitem[Vogt et~al.(2020)Vogt, Hendrix, and Nenzi]{vogt2020ngspice}
H.~Vogt, M.~Hendrix, and P.~Nenzi.
\newblock Ngspice (version 31), 2020.
\newblock URL \url{http://ngspice.sourceforge.net/docs/ngspice-manual.pdf}.

\bibitem[Wang et~al.(2019)Wang, Li, Song, Rao, Belkin, Li, Yan, Jiang, Lin, Hu,
  et~al.]{wang2019reinforcement}
Z.~Wang, C.~Li, W.~Song, M.~Rao, D.~Belkin, Y.~Li, P.~Yan, H.~Jiang, P.~Lin,
  M.~Hu, et~al.
\newblock Reinforcement learning with analogue memristor arrays.
\newblock \emph{Nature Electronics}, 2\penalty0 (3):\penalty0 115--124, 2019.

\bibitem[Whittington and Bogacz(2019)]{whittington2019theories}
J.~C. Whittington and R.~Bogacz.
\newblock Theories of error back-propagation in the brain.
\newblock \emph{Trends in cognitive sciences}, 2019.

\bibitem[Xia and Yang(2019)]{xia2019memristive}
Q.~Xia and J.~J. Yang.
\newblock Memristive crossbar arrays for brain-inspired computing.
\newblock \emph{Nature materials}, 18\penalty0 (4):\penalty0 309--323, 2019.

\bibitem[Zoppo et~al.(2020)Zoppo, Marrone, and Corinto]{zoppo2020equilibrium}
G.~Zoppo, F.~Marrone, and F.~Corinto.
\newblock Equilibrium propagation for memristor-based recurrent neural
  networks.
\newblock \emph{Frontiers in Neuroscience}, 14:\penalty0 240, 2020.

\end{thebibliography}

\newpage
\appendix

\part*{Appendix}

\section{Generalisation of Theorem \ref{thm:gradients} and Proof}
\label{sec:proof}

In this appendix we prove Theorem \ref{thm:gradients}, which we recall here. Theorem \ref{thm:gradients} holds for any \textit{nonlinear resistive network}, that is any circuit consisting of interconnected two-terminal components for which the current-voltage characteristics (defined in section \ref{sec:parameter-gradient}) are well defined and continuous.

\thmgradients*

\paragraph{Sketch of the proof.}
We prove Theorem \ref{thm:gradients} in three steps.
\begin{enumerate}
\item In Section \ref{sec:parameter-gradient}, we state a more general formula for computing the gradient with respect to a parameter of an arbitrary two-terminal component (Theorem \ref{lma:gradients}), and we show that the formula of Theorem \ref{thm:gradients} is a particular case of Theorem \ref{lma:gradients}. More specifically, for every two-terminal component we define a quantity called the pseudo-power of the component. Theorem \ref{lma:gradients} states that the gradient of the loss with respect to a parameter of this component can be expressed in terms of its pseudo-power. In the case of a linear resistor we recover the formula of Theorem \ref{thm:gradients}. (At this stage it then remains to prove Theorem \ref{lma:gradients}.)
\item In Section \ref{sec:general-theorem}, following the work of \citet{johnson2010nonlinear}, we show that in a nonlinear resistive network, the steady state of the circuit imposed by Kirchhoff's laws is a critical point of the total pseudo-power of the circuit (Lemma \ref{lma:power}), which by definition is the sum of the pseudo-powers of its individual components. In this sense we say that nonlinear resistive networks are energy-based models (EBMs) whose energy function is the total pseudo-power.
This result bridges a conceptual gap between energy functions in EBMs (at a mathematical level) and physical energies (at a hardware level). (At this stage it then remains to prove Theorem \ref{lma:gradients}, given Lemma \ref{lma:power}.)
\item In Section \ref{sec:proof-general-theorem}, we follow the general method of \citet{Scellier+Bengio-frontiers2017} to prove Theorem \ref{lma:gradients} using Lemma \ref{lma:power}. Specifically, we proceed in three sub-steps. First, by denoting $\mathcal{P}$ the total pseudo-power of the circuit in the free phase (first phase of training), and $\ell$ the cost function, we show that the total pseudo-power of the circuit in the nudged phase (second phase of training, when we inject currents $I_k = - \beta \frac{\partial \ell}{\partial \widehat{Y}_k}$ at output nodes) is equal to $\mathscr{P} = \mathcal{P} + \beta \ell$. Second, we use the criticality condition of Lemma \ref{lma:power} to state and prove a result about $\mathscr{P}$ (Lemma \ref{lma:aux}). Third, we show that Theorem \ref{lma:gradients} is a consequence of Lemma \ref{lma:aux}.
\end{enumerate}

\subsection{Gradient With Respect to a Parameter of an Arbitrary Two-Terminal Component}
\label{sec:parameter-gradient}

In this subsection, we state Theorem \ref{lma:gradients}, and we show that Theorem \ref{thm:gradients} is a special case of it. Theorem \ref{lma:gradients} states that the gradient of the loss with respect to the parameter of an arbitrary two-terminal component can be expressed in terms of a quantity called the \textit{pseudo-power} of the component. The pseudo-power is itself defined in terms of the current-voltage characteristic of the component.

\paragraph{Current-voltage characteristic of a two-terminal component.}
Consider a two-terminal component (such as a resistor or a diode) with end nodes $i$ and $j$. Its behaviour is determined by its current-voltage characteristic, which is a function $\gamma_{ij}$ that takes as input the voltage drop $\Delta V_{ij} = V_i - V_j$ across the component and returns the current $I_{ij} = \gamma_{ij} \left( \Delta V_{ij} \right)$ moving from node $i$ to node $j$ in response to the voltage drop $\Delta V_{ij}$. Since the current flowing from $i$ to $j$ is the negative of the current flowing from $j$ to $i$, we have by definition:
\begin{equation}
    \label{eq:antisymmetry}
    \forall i,j, \qquad \gamma_{ij} \left( \Delta V_{ij} \right) = - \gamma_{ji} \left( \Delta V_{ji} \right)
\end{equation}
where $\Delta V_{ji} = - \Delta V_{ij}$.

For example, the current-voltage characteristic of a linear resistor of conductance $g_{ij}$ linking node $i$ to node $j$ is given by Ohm's law: $I_{ij} = g_{ij} \Delta V_{ij}$, which by definition of $\gamma_{ij}$ implies that
\begin{equation}
    \label{eq:IVresistor}
    \gamma_{ij} \left( \Delta V_{ij} \right) = g_{ij} \Delta V_{ij}.
\end{equation}

\paragraph{Pseudo-power of a two-terminal component.}
For each two-terminal component of current-voltage characteristic $I_{ij} = \gamma_{ij}(\Delta V_{ij})$, we define $p_{ij}(\Delta V_{ij})$ as the primitive function of $\gamma_{ij}(\Delta V_{ij})$ that vanishes at $0$, i.e.
\begin{equation}
    \label{eq:pseudo-power}
    p_{ij}(\Delta V_{ij}) = \int_0^{\Delta V_{ij}} \gamma_{ij}(v) dv.
\end{equation}
The quantity $p_{ij} \left( \Delta V_{ij} \right)$ has the physical dimensions of power, being a product of a voltage and a current. We call $p_{ij} \left( \Delta V_{ij} \right)$ the \textit{pseudo-power} along the branch from $i$ to $j$, following the terminology of \citet{johnson2010nonlinear}. Note that as a consequence of Eq.~\ref{eq:antisymmetry} we have
\begin{equation}
    \label{eq:symmetry}
    \forall i,j, \qquad p_{ij}(\Delta V_{ij}) = p_{ji}(\Delta V_{ji}),
\end{equation}
i.e. the pseudo-power from $i$ to $j$ and the pseudo-power from $j$ to $i$ are the same thing. We call this property the \textit{pseudo-power symmetry}.

For example, in the case of a linear resistor of conductance $g_{ij}$ linking node $i$ to node $j$, the current-voltage characteristic $\gamma_{ij}(\Delta V_{ij})$ is given by Eq.~\ref{eq:IVresistor}, and its corresponding pseudo-power is:
\begin{equation}
\label{eq:pseudo-resistor}
p_{ij}(\Delta V_{ij}) = \frac{1}{2} g_{ij} \Delta V_{ij}^2.
\end{equation}
In this case, the pseudo-power is half the physical power dissipated in the resistor.

\begin{restatable}{thm}{lmagradients}
\label{lma:gradients}
Consider a nonlinear resistive network, and let $w_{ij}$ denote an adjustable parameter of a two-terminal component whose terminals are $i$ and $j$. Denote $\Delta V_{ij}^0$ the voltage drop across this two-terminal component in the free phase (when no current is sourced at output nodes), and $\Delta V_{ij}^\beta$ the voltage drop in the nudged phase (when a current $I_k = - \beta \frac{\partial \ell}{\partial \widehat{Y}_k}$ is sourced at each output node $\widehat{Y}_k$). Then, the gradient of the loss ${\mathcal L}  = \ell \left( \widehat{Y},Y \right)$ with respect to $w_{ij}$ can be estimated, in the limit $\beta \to 0$, as
\begin{equation}
\label{eq:thm-ep}
\frac{\partial {\mathcal L}}{\partial w_{ij}} =
\lim_{\beta \to 0} \frac{1}{\beta} \left( \frac{\partial p_{ij} \left( \Delta V^\beta_{ij} \right)}{\partial w_{ij}} -  \frac{\partial p_{ij} \left( \Delta V^0_{ij} \right)}{\partial w_{ij}} \right).
\end{equation}
\end{restatable}

In the case of a resistor of conductance $g_{ij}$, the adjustable parameter is $w_{ij} = g_{ij}$ and the pseudo-power $p_{ij} \left( \Delta V_{ij} \right)$ is given by Eq.~\ref{eq:pseudo-resistor}, so that
\begin{equation}
    \frac{\partial p_{ij} \left( \Delta V_{ij} \right)}{\partial g_{ij}} = \frac{1}{2} \left( \Delta V_{ij} \right)^2.
\end{equation}
Thus, Theorem \ref{thm:gradients} is a special case of Theorem \ref{lma:gradients}. It remains to prove Theorem \ref{lma:gradients}, which we do next.

\subsection{Nonlinear Resistive Networks are Energy-Based Models}
\label{sec:general-theorem}

In this subsection we show that, in a nonlinear resistive network, the steady state of the circuit characterized by Kirchhoff's laws is a critical point of a function called the \textit{total pseudo-power} (Lemma \ref{lma:power}). The total pseudo-power is the sum of the pseudo-powers of the individual components of the circuit.

We number the nodes of the circuit $i=1,2,\ldots,N$ and denote the node voltages $V_1, V_2, \ldots, V_N$.

\paragraph{Configuration of the circuit.}
We call a vector of voltage values $V = \left( V_1, V_2, \ldots, V_N \right)$ a \textit{configuration}. Importantly, according to our definition, any vector of voltage values is a configuration, even those that are not compatible with the laws governing electrical circuits (Kirchhoff's current law).

\paragraph{Total pseudo-power of a configuration.}
Recall the definition of the pseudo-power of a two-terminal component (Eq.~\ref{eq:pseudo-power}). We define the \textit{total pseudo-power} of a configuration $V = \left( V_1, V_2, \ldots, V_N \right)$ as the sum of pseudo-powers along all branches:
\begin{equation}
    \label{eq:total-pseudo-power}
    \mathcal{P}(V_1, \cdots, V_N) = \sum_{i<j} p_{ij}(V_i - V_j).
\end{equation}
Notice that the pseudo-power symmetry (Eq.~\ref{eq:symmetry}) guarantees that this definition does not depend on node ordering.

Consider as an example the case of a linear resistance network, i.e. a circuit composed of nodes interconnected by linear resistors. The pseudo-power of each individual resistor is given by Eq.~\ref{eq:pseudo-resistor} and is half the power dissipated by the resistor. Therefore the total pseudo-power of the circuit is half the total power dissipated by the circuit:
\begin{equation}
    \mathcal{P}(V_1,V_2,\ldots,V_N) = \frac{1}{2} \sum_{i<j} g_{ij} \left(V_j-V_i \right)^2.
    \label{eq:power-linear-resistance-network}
\end{equation}

We stress that $\mathcal{P}$ is a mathematical function defined on any configuration $V_1, V_2, \ldots, V_N$, even those that are not compatible with Kirchhoff's current law.

\paragraph{Steady state of the circuit.}
The configuration of the circuit that is physically realized is imposed by Kirchhoff's current law (KCL). We denote $V_1^\star$, $V_2^\star$, $\ldots$, $V_N^\star$ the voltage values imposed by KCL, and we call $V^\star = \left( V_1^\star, V_2^\star, \ldots, V_N^\star \right)$ the \textit{steady state} of the circuit. Specifically, for every (internal or output) floating node $i$, KCL implies $\sum_{j=1}^N I_{ij} = 0$, which rewrites
\begin{equation}
    \label{eq:stationary}
    \sum_{j=1}^N \gamma_{ij} \left( V_i^\star-V_j^\star \right) = 0.
\end{equation}

\begin{restatable}[\citet{johnson2010nonlinear}]{lma}{lmacritical}
\label{lma:power}
The steady state of the circuit, denoted $\left( V_1^\star, V_2^\star, \ldots, V_N^\star \right)$, is a critical point\footnote{With further assumptions on the current-voltage functions $\gamma_{ij}$, \citet{christianson2007dirichlet} and \citet{johnson2010nonlinear} show that the function $\mathcal{P}$ is convex, so that the steady state is the global minimum of $\mathcal{P}$. In our case however, in order to prove Theorem \ref{lma:gradients}, all one needs in the framework of Equilibrium Propagation is the first order condition, i.e. the fact that the steady state is a critical point of $\mathcal{P}$, not necessarily a global or local minimum.}
of the total pseudo-power: for every floating node $i$, we have
\begin{equation}
    \frac{\partial \mathcal{P}}{\partial V_i} \left( V_1^\star, V_2^\star, \ldots, V_N^\star \right) = 0.
\end{equation}
\end{restatable}

We say in this sense that the circuit is an energy-based model, whose energy function is the total pseudo-power.

\begin{proof}[Proof of Lemma \ref{lma:power}]
We use the definition of the total pseudo-power (Eq.~\ref{eq:total-pseudo-power}), the pseudo-power symmetry (Eq.~\ref{eq:symmetry}), the definition of the pseudo-power (Eq.~\ref{eq:pseudo-power}) and the fact that the steady state of the circuit satisfies Kirchhoff's current law (Eq.~\ref{eq:stationary}). For every floating node $i$ we have:
\begin{align}
\frac{\partial \mathcal{P}}{\partial V_i} \left( V_1^\star, V_2^\star, \ldots, V_N^\star \right) = & \sum_j \frac{\partial p_{ij}}{\partial V_i}(V_i^\star-V_j^\star)\\
    = & \sum_j \gamma_{ij}(V_i^\star-V_j^\star) = 0.
\end{align}
\end{proof}

Lemma \ref{lma:power} generalizes a well-known property of linear resistance networks (i.e. circuits composed of linear resistors) called the \textit{principle of minimum dissipated power} \citep{baez2015compositional}. This principle states that in a linear resistance network, if the voltages are imposed at a set of input nodes, the circuit will choose the voltages at other nodes so as to minimize the total power dissipated in the resistors (i.e. minimize Eq.~\ref{eq:power-linear-resistance-network}).

\subsection{Proof of Theorem \ref{lma:gradients}}
\label{sec:proof-general-theorem}

Equilibrium Propagation \citep{Scellier+Bengio-frontiers2017} applies to models that possess an energy function, in the sense of Lemma \ref{lma:power}. Equipped with Lemma \ref{lma:power}, we can therefore prove Theorem \ref{lma:gradients} by following the general method of \citet{Scellier+Bengio-frontiers2017}. We do this in three steps:
\begin{itemize}
\item First, we rewrite the formula of Theorem \ref{lma:gradients} in terms of the total pseudo-power of the first phase (free phase), denoted $\mathcal{P}$, and the cost function $\ell$.
\item Second, we introduce the total pseudo-power of the second phase (nudged phase), denoted $\mathscr{P}$, which we show to be equal to $\mathscr{P} = \mathcal{P} + \beta \ell$, and we rewrite Theorem \ref{lma:gradients} in terms of $\mathscr{P}$.
\item Third, we prove Theorem \ref{lma:gradients} using Lemma \ref{lma:power}.
\end{itemize}

\paragraph{First step. Rewriting Theorem \ref{lma:gradients} in terms of the total pseudo-power of the first phase $\mathcal{P}$ and the cost function $\ell$.}

Let $V$ be the configuration of floating node voltages, which includes the internal nodes and output nodes. $\theta$ denotes the vector of adjustable parameters (conductances), and $\mathcal{P}(X,V,\theta)$ denotes the total pseudo-power of the circuit in the first phase of training (free phase).

Consider an adjustable parameter $w_{ij}$ of a component whose terminals are $i$ and $j$. This parameter contributes to $\mathcal{P}(X,V,\theta)$ only through the pseudo-power of the component it belongs to, i.e. only through the term $p_{ij}(V_i - V_j)$.
Thus, we have (see Eq.~\ref{eq:total-pseudo-power})
\begin{equation}
    \label{eq:dP-dw}
   \frac{\partial \mathcal{P}(X,V,\theta)}{\partial w_{ij}} = \frac{\partial p_{ij}(V_i - V_j)}{\partial w_{ij}}.
\end{equation}
Therefore, the formula to be proved (Theorem \ref{lma:gradients}) rewrites:
\begin{equation}
\frac{\partial {\mathcal L}(X,Y,\theta)}{\partial w_{ij}}
= \lim_{\beta \to 0} \frac{1}{\beta} \left( \frac{\partial \mathcal{P} \left( X,V^\beta,\theta \right)}{\partial w_{ij}} -  \frac{\partial \mathcal{P} \left( X,V^0,\theta \right)}{\partial w_{ij}} \right).
\label{eq:to-prove1}
\end{equation}

Here $V^0$ and $V^\beta$ are the steady states of the circuit in the first phase (free phase) and second phase (nudged phase), respectively. Since the formula of Eq.~\ref{eq:to-prove1} is to be proved for all parameters $w_{ij}$ of the circuit, we can also rewrite Theorem \ref{lma:gradients} in the more compact form:
\begin{equation}
\frac{\partial {\mathcal L}}{\partial \theta}(X,Y,\theta)
= \lim_{\beta \to 0} \frac{1}{\beta} \left( \frac{\partial \mathcal{P}}{\partial \theta}\left( X,V^\beta,\theta \right) -  \frac{\partial \mathcal{P}}{\partial \theta}\left( X,V^0,\theta \right) \right).
\label{eq:to-prove2}
\end{equation}

Recall that the loss to be optimized (Eq.~\ref{eq:loss}) is
\begin{equation}
{\mathcal L}(X,Y,\theta) = \ell(V^0,Y).
\end{equation}
Note the difference between the loss ${\mathcal L}$ and the cost function $\ell$: the cost is defined for any configuration $V$, whereas the loss is the cost value at the steady state $V^0$. Recall here that ${\mathcal L}$ depends on $X$ and $\theta$ through $V^0$. Now, the formula to be proved (Eq.~\ref{eq:to-prove2}) rewrites:
\begin{equation}
\frac{d}{d\theta} \ell(V^0,Y)
= \lim_{\beta \to 0} \frac{1}{\beta} \left( \frac{\partial \mathcal{P}}{\partial \theta}\left( X,V^\beta,\theta \right) -  \frac{\partial \mathcal{P}}{\partial \theta}\left( X,V^0,\theta \right) \right).
\label{eq:to-prove3}
\end{equation}
Differentiation with respect to $\theta$ on the left-hand side is to be understood through $V^0$.

\paragraph{Second step. Rewriting Theorem \ref{lma:gradients} in terms of the total pseudo-power of the second phase $\mathscr{P}$.} Recall that $\mathcal{P}(X,V,\theta)$ is the total pseudo-power in the first phase of training (free phase). In the second phase of training though (nudged phase), a current $I_k = -\gamma_k^\beta(V)$ is sourced at every floating node\footnote{In section \ref{sec:electrical-networks}, we have considered a cost function $\ell(\widehat{Y},Y)$ which depends only on output node voltages $\widehat{Y}$, not on internal node voltages. In this context, in the second phase of EqProp we only source currents at output nodes. However, this setting can be directly generalized to the case of a cost function $\ell(V,Y)$ that depends on internal node voltages too.} $V_k$, where
\begin{equation}
    \gamma_k^\beta(V) = \beta  \; \frac{\partial \ell}{\partial V_k}(V,Y)
\end{equation}
represents the current flowing from node $k$ outwards (towards the current source), in agreement with the sign convention in the definition of the current-voltage characteristic adopted in section \ref{sec:parameter-gradient}. We can then define
\begin{equation}
    p^\beta(V) = \beta  \; \ell(V,Y),
\end{equation}
which plays the role of pseudo-power associated to all the currents $I_k$ as it satisfies:
\begin{equation}
\forall k, \qquad \gamma_k^\beta(V) = \frac{\partial p^\beta}{\partial V_k}(V).
\end{equation}
Thus, in the second phase of training (nudged phase), the total pseudo-power of the circuit is
\begin{equation}
    \label{eq:relationship}
    \mathscr{P}(\theta,\beta,V) = \mathcal{P}(X,V,\theta) + \beta \; \ell(V,Y).
\end{equation}
Subsequently, we drop the input $X$ and target $Y$ in the notations, since they are static and do not play any role in the proof of Theorem \ref{lma:gradients}. With these notations the total pseudo-power of the first phase (free phase) is $\mathscr{P}(\theta,0,V)$, and the total pseudo-power of the second phase (nudged phase) is $\mathscr{P}(\theta,\beta,V)$.

From now on we also rewrite $V_\theta^0$ and $V_\theta^\beta$ the steady states of the circuit in the first phase (free phase) and second phase (nudged phase), respectively, to stress that they depend on the parameter $\theta$. From Lemma \ref{lma:power}, $V_\theta^0$ is a critical point of $\mathscr{P}(\theta,0,\cdot)$, and $V_\theta^\beta$ is a critical point of $\mathscr{P}(\theta,\beta,\cdot)$. With these notations, the formula to be proved (Eq.~\ref{eq:to-prove3}) rewrites:
\begin{equation}
\frac{d}{d\theta} \frac{\partial \mathscr{P}}{\partial \beta}(\theta,0,V_\theta^0)
= \lim_{\beta \to 0} \frac{1}{\beta} \left( \frac{\partial \mathscr{P}}{\partial \theta}\left( \theta,\beta,V_\theta^\beta \right) -  \frac{\partial \mathscr{P}}{\partial \theta}\left( \theta,0,V_\theta^0 \right) \right),
\label{eq:to-prove4}
\end{equation}
or more compactly\footnote{Note that the notations $\frac{d}{d\theta}$ and $\frac{\partial}{\partial \theta}$ (as well as $\frac{d}{d\beta}$ and $\frac{\partial}{\partial \beta}$) have two different meanings. The notation $\frac{\partial \mathscr{P}}{\partial \theta}(\theta,\beta,V)$ represents the \textit{partial derivative} of $\mathscr{P}$ with respect to its first argument ($\theta$) ; this excludes differentiation through $V_\theta^\beta$. The notation $\frac{d}{d\theta} f(\theta,\beta,V_\theta^\beta)$ represents the \textit{total derivative} of the expression with respect to $\theta$: this includes the differentiation paths through both the first argument of $f$ and through $V_\theta^\beta$.}:
\begin{equation}
\frac{d}{d\theta} \frac{\partial \mathscr{P}}{\partial \beta}(\theta,0,V_\theta^0) = \left. \frac{d}{d\beta} \right|_{\beta=0} \frac{\partial \mathscr{P}}{\partial \theta}(\theta,\beta,V_\theta^\beta),
\label{eq:to-prove5}
\end{equation}
where $\left. \frac{d}{d\beta} \right|_{\beta=0}$ represents the total derivative of the expression, evaluated at the point $\beta=0$.

\paragraph{Third step. Statement and Proof of Lemma \ref{lma:aux}.}

The formula to be proved (Eq.~\ref{eq:to-prove5}) is a direct consequence of the following result proved in \citet{Scellier+Bengio-frontiers2017}, which we also prove here for completeness.

\begin{lma}[\citet{Scellier+Bengio-frontiers2017}]
\label{lma:aux}
For any value of $\beta$, we have the relationship:
\begin{equation}
\frac{d}{d\theta} \frac{\partial \mathscr{P}}{\partial \beta}(\theta,\beta,V_\theta^\beta) = \frac{d}{d\beta} \frac{\partial \mathscr{P}}{\partial \theta}(\theta,\beta,V_\theta^\beta).
\end{equation}
\end{lma}

\begin{proof}[Proof of Lemma \ref{lma:aux}]
Let us define a function $H$ of the two arguments $\theta$ and $\beta$:
\begin{equation}
  H \left( \theta,\beta \right) = \mathscr{P} \left(\theta,\beta,V_\theta^\beta \right).
\end{equation}
Note that $H$ is a function of $(\theta,\beta)$ not only through $\mathscr{P}(\theta,\beta,\cdot)$ but also through $V_\theta^\beta$. Using the chain rule of differentiation and Lemma \ref{lma:power} we have
\begin{align}
  \frac{\partial H}{\partial \beta}(\theta,\beta)
  & = \frac{\partial \mathscr{P}}{\partial \beta} \left( \theta,\beta,V_\theta^\beta \right)
  + \underbrace{\frac{\partial \mathscr{P}}{\partial V} \left( \theta,\beta,V_\theta^\beta \right)}_{= \; 0} \cdot \frac{\partial V_\theta^\beta}{\partial \beta} \\
  & = \frac{\partial \mathscr{P}}{\partial \beta} \left( \theta,\beta,V_\theta^\beta \right).
\end{align}
Differentiating this expression with respect to $\theta$, we get
\begin{equation}
  \frac{\partial^2 H}{\partial \theta \partial \beta}(\theta,\beta)
  = \frac{d}{d\theta} \frac{\partial \mathscr{P}}{\partial \beta} \left( \theta,\beta,V_\theta^\beta \right).
\end{equation}
Similarly, we have that
\begin{align}
  \frac{\partial H}{\partial \theta}(\theta,\beta)
  & = \frac{\partial \mathscr{P}}{\partial \theta} \left( \theta,\beta,V_\theta^\beta \right)
  + \underbrace{\frac{\partial \mathscr{P}}{\partial V} \left( \theta,\beta,V_\theta^\beta \right)}_{= \; 0} \cdot \frac{\partial V_\theta^\beta}{\partial \theta} \\
  & = \frac{\partial \mathscr{P}}{\partial \theta} \left( \theta,\beta,V_\theta^\beta \right),
\end{align}
so that
\begin{equation}
  \frac{\partial^2 H}{\partial \beta \partial \theta}(\theta,\beta)
  = \frac{d}{d\beta} \frac{\partial \mathscr{P}}{\partial \theta} \left( \theta,\beta,V_\theta^\beta \right).
\end{equation}
Finally, we conclude using the symmetry of second derivatives of $H$:
\begin{equation}
  \frac{\partial^2 H}{\partial \theta \partial \beta} \left( \theta,\beta \right)
  = \frac{\partial^2 H}{\partial \beta \partial \theta} \left( \theta,\beta \right).
\end{equation}
\end{proof}

\clearpage
\section{General Applicability of EqProp (Complement of Section \ref{sec:eqprop})}
\label{sec:general-applicability}

In section \ref{sec:eqprop} we have presented the setting of classification and regression tasks where the goal is to predict a target $Y$ given an input $X$.

In this appendix, we show the general applicability of EqProp.
In section \ref{sec:loss}, we give examples of loss functions and the corresponding gradient currents to be sourced in the second phase of training (nudged phase).
We then show in section \ref{sec:gan} that EqProp extends well beyond the setting of classification and regression tasks, and can also be used in the setting of generative adversarial networks \citep{goodfellow2014generative}.

\subsection{Loss Functions}
\label{sec:loss}

In the supervised setting, the goal of training is to minimize the expected loss
\begin{equation}
J(\theta) = \mathbb{E}_{(X,Y) \sim p(X,Y)} \left[ \mathcal{L}(X,Y,\theta) \right],
\end{equation}
over input-target pairs $(X,Y)$ drawn from a data distribution $p(X,Y)$. Recall that the loss $\mathcal{L}$ of Eq.~\ref{eq:loss} for a given input-target pair $(X,Y)$ is of the form:
\begin{equation}
\label{eq:loss-repeat}
\mathcal{L}(X,Y,\theta) = \ell \left( \widehat{Y},Y \right),
\end{equation}
where $\widehat{Y}$ is the steady state of output nodes at inference, and $\ell(\widehat{Y},Y)$ is an arbitrary differentiable cost function.

We study here two examples of common cost functions $\ell$ and give the corresponding currents to be sourced in the second phase of training (nudged phase).

\subsubsection{Squared Error Loss}

The most straightforward example is the squared error cost function, given by
\begin{equation}
\ell(\widehat{Y},Y) = \frac{1}{2} \sum_{k=1}^K \left( \widehat{Y}_k-Y_k \right)^2,
\end{equation}
which sums the quadratic distances between each output voltage $\widehat{Y}_k$ and its corresponding target voltage $Y_k$.
In this case the error derivatives of output nodes are:
\begin{equation}
    \frac{\partial \ell}{\partial \widehat{Y}_k}(\widehat{Y},Y) = \widehat{Y}_k - Y_k,
\end{equation}
and the corresponding current $I_k$ to source at output node $\widehat{Y}_k$ in the second phase of training is
\begin{equation}
    I_k = \beta \left( Y_k - \widehat{Y}_k \right).
\end{equation}

\subsubsection{Softmax and Cross-Entropy Loss}

In a classification task, the target vector $Y = (Y_1,Y_2, \ldots, Y_K)$ represents the one-hot code of the class label. In order to apply the cross-entropy loss, we view the vector of output voltages $\widehat{Y} = (\widehat{Y}_1, \widehat{Y}_2, \ldots, \widehat{Y}_K)$ as the vector of \textit{logits}, i.e. the vector of scores assigned to each of the $K$ classes. The vector of class probabilities associated to these logits is $p = (p_1, p_2, \ldots, p_K)$, where 
\begin{equation}
\forall k=1,2,\ldots,K, \qquad p_k = \frac{\exp(\widehat{Y}_k)}{\sum_{i=1}^K \exp(\widehat{Y}_i)}.
\end{equation}
We write for short $p = \textrm{softmax}(\widehat{Y})$ the vector of class probabilities, and $p_k = \textrm{softmax}_k(\widehat{Y})$ the probability of class $k$. The cost function associated to the cross-entropy loss is
\begin{equation}
\ell(\widehat{Y},Y) = - \sum_{k=1}^K Y_k \log(\textrm{softmax}_k(\widehat{Y})).
\end{equation}
In this case the error derivatives of output nodes (logits) are given by:
\begin{equation}
\forall k, \qquad \frac{\partial \ell}{\partial \widehat{Y}_k}(\widehat{Y},Y) = \textrm{softmax}_k(\widehat{Y}) - Y_k,
\end{equation}
so the corresponding current $I_k$ to be sourced at output node $\widehat{Y}_k$ in the second phase of EqProp is
\begin{equation}
I_k = \beta \left( Y_k - p_k \right).
\end{equation}
In practical neuromorphic hardware, the vector of class probabilities $p = (p_1, p_2, \ldots, p_K)$ can be computed digitally, and then the currents $I_k = \beta \left( Y_k - p_k \right)$ can be injected at output nodes with current sources.

\subsubsection{Binary Cross-Entropy}
\label{sec:binary-cross-entropy}

A special important case of the softmax/cross-entropy loss defined above is the binary cross-entropy, which is useful in the setting of GANs (section \ref{sec:gan}).

In this case there is only one output node, i.e. $\widehat{Y}$ is a scalar, which we view as the logit. The probability associated to this logit is the sigmoid, defined as:
\begin{equation}
\sigma(\widehat{Y}) = \frac{1}{1 + \exp(-\widehat{Y})}.
\end{equation}
The binary cross-entropy cost function is then defined as:
\begin{equation}
\ell_{BCE}(\widehat{Y},Y) = - Y \log(\sigma(\widehat{Y})) - (1-Y) \log(1-\sigma(\widehat{Y})).
\end{equation}
The gradient with respect to the unique output node is
\begin{equation}
\frac{\partial \ell_{BCE}}{\partial \widehat{Y}}(\widehat{Y},Y) = \sigma(\widehat{Y}) - Y,
\end{equation}
and the corresponding current injected in the second phase of EqProp is
\begin{equation}
I = \beta \left( Y - \sigma(\widehat{Y}) \right).
\end{equation}

\subsubsection{Weight Decay}

One can easily adapt the update rule of Theorem \ref{thm:gradients} to include a weight regularization term.
Suppose that, instead of the loss of Eq.~\ref{eq:loss-repeat}, we want to optimize a loss of the form
\begin{equation}
    \mathcal{L}(X,Y,\theta) = \ell \left( \widehat{Y},Y \right) + \Omega(\theta),
\end{equation}
where $\Omega(\theta)$ is a weight regularization term. Typically, $\Omega(\theta)$ is the sum of squared weights:
\begin{equation}
    \Omega(\theta) = \frac{\lambda}{2} \sum_{i,j} g_{ij}^2,
\end{equation}
where the $g_{ij}$'s are the conductances of the programmable resistors. In order to compute the gradient of $\mathcal{L}$, the formula of Theorem \ref{thm:gradients} needs to be slightly changed by adding the term $\frac{\partial \Omega}{\partial g_{ij}}$, which here is equal to $\lambda g_{ij}$. We get:
\begin{equation}
    \frac{\partial \mathcal{L}}{\partial g_{ij}} = \lim_{\beta \to 0} \frac{1}{\beta} \left[ \left( \Delta V^\beta_{ij} \right)^2 -  \left( \Delta V^0_{ij} \right)^2 \right] + \lambda g_{ij}.
\end{equation}
Note that the currents injected in the second phase of EqProp are unchanged, i.e. $I_k = - \frac{\partial \ell}{\partial \widehat{Y}_k}$ for each output node $\widehat{Y}_k$.

\subsection{Generative Adversarial Networks}
\label{sec:gan}

In this section, we show how the setting of Generative Adversarial Networks (GANs) \citep{goodfellow2014generative} can be applied to analog neural networks trained with EqProp.

\subsubsection{Description of the GAN setting}

In the setting of GANs, a generative network (or simply `generator', denoted $G$) and a discriminative network (or simply `discriminator', denoted $D$) are trained concurrently. The goal is to have the generator $G$ produce samples (denoted $\widehat{X}$) with similar statistical properties as the data samples (denoted $X$) of a data distribution $p(X)$. To achieve this, the discriminator $D$ is trained to distinguish samples $\widehat{X}$ produced by $G$ from true samples $X$ coming from $p(X)$. Simultaneously, the generator $G$ is trained to maximise the probability of $D$ making a mistake, i.e. $G$ is trained to `fool' the discriminator by producing samples that the discriminator thinks are part of the true data distribution $p(X)$.

More formally, the generator $G$ is parameterized by a set of weights $\theta$ and implements a function $\widehat{X} = G_\theta(Z)$, where $Z$ is a latent variable coming from a known distribution $p(Z)$. The discriminator $D$ is parameterized by a set of weights $\phi$ and implements a function $\widehat{Y} = D_\phi(\widetilde{X})$, where $\widetilde{X}$ is either generated by $G$ or coming from the true $p(X)$, and $\widehat{Y}$ is the probability of $\widetilde{X}$ coming from $p(X)$. The discriminator and the generator are trained to minimize the following two objectives $J_D$ and $J_G$ with respect to $\phi$ and $\theta$, respectively:
\begin{align}
    \label{eq:objective-discriminator}
    J_D(\phi) = & - \mathbb{E}_{X \sim p(X)} \left[ \log(D_\phi(X)) \right] - \mathbb{E}_{Z \sim p(Z)} \left[ \log(1-D_\phi(G_\theta(Z))) \right],\\
    J_G(\theta) = & - \mathbb{E}_{Z \sim p(Z)} \left[  \log(D_\phi(G_\theta(Z))) \right].
    \label{eq:objective-generator}
\end{align}

\subsubsection{Training the Generator}

In the generator $G$, a set of nodes are the `latent nodes' $Z$ whose voltage values are sampled from $p(Z)$ and sourced. Another set of nodes are the `data nodes' $\widehat{X}$ which correspond to the generated sample.

The objective function for the generator (Eq.~\ref{eq:objective-generator}) rewrites $J_G(\theta) = \mathbb{E}_{Z \sim p(Z)} \left[ \mathcal{L}_G(Z,\theta) \right]$, with
\begin{align}
    \mathcal{L}_G(Z,\theta) & = \ell_G(G_\theta(Z)), \\
    \ell_G(\widehat{X}) & =  - \log(D_\phi(\widehat{X})).
    \label{eq:ell_G}
\end{align}
Having defined the loss $\mathcal{L}_G$ and the cost function $\ell_G$ corresponding to the generator $G$, we can then use EqProp to train $G$ on a given latent sample $Z \sim p(Z)$, provided that for each data node $\widehat{X}_i$ of $G$ we can compute the current $I_i$ to be injected at $\widehat{X}_i$ in the second phase (nudged phase). These currents to be sourced in the nudged phase are equal to:
\begin{equation}
    I_i = - \beta \; \frac{\partial \ell_G}{\partial \widehat{X}_i}(\widehat{X}).
    \label{eq:currents}
\end{equation}
Thus we need to compute $\frac{\partial \ell_G}{\partial \widehat{X}_i}(\widehat{X})$. To do this, let us rewrite the function $\ell_G(\widehat{X})$ of Eq.~\ref{eq:ell_G} as
\begin{equation}
\label{eq:discrimin}
\ell_G(\widehat{X}) = \ell_{BCE}(D_\phi(\widehat{X}),1),
\end{equation}
where $\ell_{BCE}$ is the binary cross-entropy\footnote{For more readability we have chosen here to denote $\widehat{Y} = D_\phi(\widehat{X})$ as the \textit{probability} of $D$ classifying $\widehat{X}$ as true data, to be consistent with the notations used in the GAN literature. However in an analog network, $\widehat{Y} = D_\phi(\widehat{X})$ would actually be the \textit{logit}, and the probability would be $\sigma(\widehat{Y})$ where $\sigma$ is the sigmoid function.} cost function presented in section \ref{sec:binary-cross-entropy}, with target $Y=1$ (the generator $G$ is trained to trick $D$ to classify the sample $\widehat{X}$ as a true data sample).
We see now from Eq.~\ref{eq:discrimin} that $\frac{\partial \ell_G}{\partial \widehat{X}_i}(\widehat{X})$ represents the gradient with respect to the input nodes of the discriminator $D$. This can be achieved using EqProp (again) in the discriminator $D$ this time, using the following Lemma.

\begin{lma}
\label{lma:inputs}
Consider an input node of the discriminator $D$. Let $\widehat{X}_i$ be the voltage sourced at this input node and $I_i$ the current flowing through it.
Denote $I_i^0$ and $I_i^\beta$ the currents in the first phase (free phase) and second phase (nudged phase) of EqProp, respectively.
Then, the gradient of $\ell_G(\widehat{X})$ with respect to $\widehat{X}_i$ is equal, in the limit $\beta \to 0$, to
\begin{equation}
\frac{\partial \ell_G}{\partial \widehat{X}_i}(\widehat{X}) =
\lim_{\beta \to 0} \frac{1}{\beta} \left( I_i^\beta - I_i^0 \right).
\end{equation}
\end{lma}

\begin{proof}
Let $\mathcal{P} \left( X,V,\phi \right)$ denote the total pseudo-power of the discriminator $D$. Recall the form of the total pseudo-power (Eq.~\ref{eq:total-pseudo-power}) and note that the current flowing through input node $X_i$ of the discriminator $D$ is
\begin{equation}
I_i = \frac{\partial \mathcal{P}}{\partial X_i}\left( X,V,\phi \right).
\end{equation}
Now we can rewrite the formula to be proved in the form:
\begin{equation}
\frac{\partial \mathcal{L}_D}{\partial X_i}(X)
= \lim_{\beta \to 0} \frac{1}{\beta} \left( \frac{\partial \mathcal{P}}{\partial X_i}\left( X,V^\beta,\phi \right) -  \frac{\partial \mathcal{P}}{\partial X_i}\left( X,V^0,\phi \right) \right),
\end{equation}
where we recall that $\ell_G = \mathcal{L}_D$.
Since this formula must be shown for all input nodes $X_i$, we can rewrite the formula to be proved more compactly as:
\begin{equation}
\frac{\partial \mathcal{L}_D}{\partial X}(X)
= \lim_{\beta \to 0} \frac{1}{\beta} \left( \frac{\partial \mathcal{P}}{\partial X}\left( X,V^\beta,\phi \right) -  \frac{\partial \mathcal{P}}{\partial X}\left( X,V^0,\phi \right) \right).
\end{equation}
The proof of this formula is identical to the proof of the formula of Eq.~\ref{eq:to-prove2}.
\end{proof}

Finally, combining Eq.~\ref{eq:currents} with Lemma \ref{lma:inputs} we see that for each data node $\widehat{X}_i$ of $G$, the required current $I_i$ to be injected at $\widehat{X}_i$ in the second phase (nudged phase) is equal to:
\begin{equation}
    I_i = I_i^\beta - I_i^0.
\end{equation}
Altogether, this gives us the following procedure to optimize the generator $G$ by SGD.

\paragraph{Equilibrium Propagation to Train the Generator of a GAN.}
In the first phase (free phase), sample $Z \sim p(Z)$, source the corresponding voltages at the $Z$ nodes of the generator $G$ and measure the voltages of all floating nodes in $G$. Then use the voltages of the $\widehat{X}$ nodes in $G$ as voltage sources for the input nodes of the discriminator $D$, and measure the currents $I_i^0$ at input nodes of $D$. This constitutes the first phase of EqProp.

In the second phase (nudged phase), source a current $I = - \beta \frac{\partial \ell_{BCE}}{\partial \widehat{Y}}$ at the output node $\widehat{Y}$ of $D$ and measure the currents $I_i^\beta$ at input nodes of $D$ anew. For each $i$, compute the difference $I_i = I_i^\beta - I_i^0$ and inject this current $I_i$ at the date node $\widehat{X}_i$ of $G$. Measure the voltages of floating nodes in $G$ anew. This constitutes the second phase of EqProp.

Finally, the gradient of $\mathcal{L}_G(Z,\theta)$ with respect to the conductances of the generator $G$ can be estimated using the formula of Theorem \ref{thm:gradients}. Specifically, let $\Delta V_{ij}$ denote the voltage drop across a programmable resistor of conductance $g_{ij}$ in the generator $G$. Denoting $\Delta V^0_{ij}$ and $\Delta V^\beta_{ij}$ the voltage drops in the first phase (free phase) and second phase (nudged phase) of EqProp, respectively, the gradient with respect to $g_{ij}$ is
\begin{equation}
    \frac{\partial \mathcal{L}_G}{\partial g_{ij}} = \lim_{\beta \to 0} \frac{1}{\beta} \left[ \left( \Delta V^\beta_{ij} \right)^2 -  \left( \Delta V^0_{ij} \right)^2 \right].
\end{equation}

\subsubsection{Training the Discriminator}

In the discriminator $D$, a set of nodes are the `input nodes' $X$ whose voltages are sourced (and correspond to either real data coming from $p(X)$ or fake data generated by $G$), and a set of nodes are the `output nodes' $\widehat{Y}$ which correspond to the probability of $X$ coming from $p(X)$.

The objective to be optimized (Eq.~\ref{eq:objective-discriminator}) can be written in the form
\begin{equation}
    J_D(\phi) =  \mathbb{E}_{X \sim p(X)} \left[ \mathcal{L}_D^1(X,\phi) \right] + \mathbb{E}_{Z \sim p(Z)} \left[ \mathcal{L}_D^2(Z,\phi) \right],
\end{equation}
where $\mathcal{L}_D^1$ and $\mathcal{L}_D^2$ are per-sample losses defined by
\begin{align}
    \mathcal{L}_D^1(X,\phi) & = \ell_{BCE}(D_\phi(X), 1), \\
    \mathcal{L}_D^2(Z,\phi) & = \ell_{BCE}(D_\phi(G_\theta(Z)), 0).
\end{align}
The function $\ell_{BCE}$ is the binary cross-entropy cost function of Section \ref{sec:binary-cross-entropy}. Computing the gradients of $\mathcal{L}_D^1(X,\phi)$ and $\mathcal{L}_D^2(Z,\phi)$ with EqProp can be done as in the supervised setting.

\paragraph{Remark.}
Note that another possibility is to have an analog generator $G$ (trained with EqProp) combined with a software-based discriminator $D$ (e.g. a multilayer perceptron) trained on conventional computer hardware with backpropagation. In this case, the currents of Eq.~\ref{eq:currents} can also be computed in software with backpropagation.

\clearpage
\section{Architecture Details (Complement of Section \ref{sec:model})}
\label{sec:model-details}

In this appendix, we provide details about the analog neural network architecture introduced in section \ref{sec:model}.

First we describe the linear resistance network model (Section \ref{sec:linear-resistance-network}). Although this model is not useful in practice, studying it is helpful to gain understanding of the working mechanisms of analog neural networks. It helps understand the limits of linear devices and the need to introduce nonlinear devices such as diodes (Section \ref{sec:diodes-details}). The linear resistance network model also helps understand the vanishing signal effect observed in our simulations (Section \ref{sec:vanishing-signal-effect}) and thus the need to introduce amplifiers (Section \ref{sec:amplifiers-details}). Another difference between conventional neural networks and analog neural networks is the fact that the weights (conductances) are all positive (Section \ref{sec:nonnegative-weights}). This structural constraint of analog networks requires practices that are uncommon in deep learning, such as doubling the number of input nodes and output nodes in the network.

\subsection{Linear Resistance Network: Insights and Limitations}
\label{sec:linear-resistance-network}

A linear resistance network is an electrical circuit whose nodes are linked pairwise by linear resistors. Linear resistance networks are extensively studied in the literature -- see e.g. \citet{baez2015compositional}.

Recall that the current-voltage relationship of a linear resistor with end node voltages $V_i$ and $V_j$ and with conductance $g_{ij}$ is given by Ohm's law: $I_{ij} = g_{ij} (V_i-V_j)$, where $I_{ij}$ is the current through the resistor. By Kirchhoff's current law, for each floating node $i$ in a linear resistance network, we have the equation: $\sum_j g_{ij} (V_i-V_j) = 0$. This rewrites:
\begin{equation}
    \label{eq:stationary-regime}
    V_i = \frac{\sum_j g_{ij} V_j}{\sum_j g_{ij}}.
\end{equation}
This operation resembles the usual multiply-accumulate operation of artificial neurons in conventional deep learning, with two notable differences: first, an additional factor $G_i = \sum_j g_{ij}$ referred to as the \textit{G term} normalizes the weighted sum of voltages; second, there is no nonlinear activation function.

Thus, in a linear resistance network, each floating node voltage $V_i$ is a weighted mean of its adjacent node voltages. Nonlinear components (such as diodes) are necessary to perform nonlinear operations.

\subsection{Current-Voltage Transfer Function of a Diode}
\label{sec:diodes-details}

The transfer function of a diode takes the following form. The current $I$ across a diode as a function of its voltage drop $\Delta V$ is given by
\begin{equation}
    I = I_S \left[ \exp \left( \frac{\Delta V}{\eta V_T} \right) - 1 \right],
\end{equation}
where $I_S$ is the reverse saturation current ($I_S \sim 1 \mu\textrm{A}$), $V_T$ is the thermal voltage ($V_T \sim 25.85 \; \textrm{mV}$ at 300 K) and $\eta$ is the emission coefficient of the diode ($\eta = 2$ for a silicon diode).

The voltage sources in series with the diodes, used to shift the bounds of the activation function, are set to $0.3$ V and $-0.7$ V.

\subsection{Vanishing Signals Effect}
\label{sec:vanishing-signal-effect}

Simulations show that in a circuit composed only of resistors and diodes (e.g. the network architecture of Figure \ref{fig:network} without the amplifiers), the signal amplitudes (i.e. node voltages) decay when going from the layer of input nodes to the first layer of hidden nodes. In the case of a linear resistance network, this \textit{vanishing signals effect} can be explained by the formula of Eq.~\ref{eq:stationary-regime}. This formula shows that each floating node voltage is a weighted mean of its adjacent node voltages, a consequence of which is that the extremal voltage values must be reached at input nodes (i.e. nodes whose voltages are sourced). Another way to see it is that the current through resistive devices always flows from high electric potential to low electric potential.

\subsection{Bidirectional Amplifier}
\label{sec:amplifiers-details}

To overcome the vanishing signal effect, we use amplifiers. Specifically, we design a bidirectional amplifier which amplifies voltages in the forward direction by a gain factor $A$, and amplifies currents in the backward direction by a gain factor $1/A$, where $A$ is a hyperparameter to be chosen. This setup is represented in Figure \ref{fig:amplifier} and can be described by the following equations:
\begin{align}
    V_2 & = A V_1, \\
    I_2 & = A I_1.
\end{align}
Note that if the gain is set to $A=1$, then the bidirectional amplifier behaves as if it were a short circuit, not influencing the steady state of the network.

Voltage amplification is performed in the forward direction using a voltage-controlled voltage source (VCVS). Specifically, the VCVS amplifies the voltage at the input of the bidirectional amplifier ($V_1$ in Fig.~\ref{fig:amplifier}) by a factor of e.g. $A=4$, resulting in an output voltage ($V_2$ in Fig.~\ref{fig:amplifier}) of $4x$ the input voltage. Current amplification is performed in the backward direction by a current-controlled current source (CCCS). Specifically, the CCCS senses the current through the VCVS and reflects it backward at the input of the bidirectional amplifier, with a gain of $1/A = 1/4$. This setup gives the bidirectional amplifier we require to allow signals to propagate bidirectionally.

\begin{figure}
\centering
\includegraphics[width=0.5\textwidth]{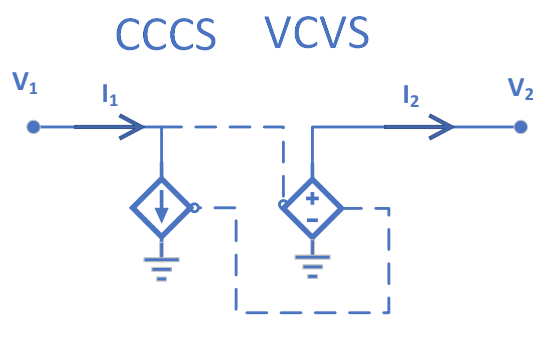}
\caption{
\textbf{Bidirectional amplifier,} composed of a voltage-controlled voltage source (VCVS) and a current-controlled current source (CCCS). The output voltage ($V_2$) is related to the input voltage ($V_1$) by the relationship $V_2 = A V_1$, where $A$ is a gain factor. The input current ($I_1$) is related to the output current ($I_2$) by the relationship $I_1 = \frac{1}{A} I_2$. In the control branches (represented by dashed lines), the current is zero.
}
\label{fig:amplifier}
\end{figure}

\subsection{Constraint of Positive Weights}
\label{sec:nonnegative-weights}

The weights in analog neural networks are represented by conductances, which are positive. This contrasts with conventional neural networks which are not subject to such constraints and whose weights are free to take either positive or negative values.

There are several ways of dealing with the constraint of positive conductances. One approach to allow weights to be either positive or negative is to decompose each weight as the difference of two (positive) conductances \citep{wang2019reinforcement}. Another approach is to shift the mean of the matrix by a constant factor as described in Section 4.1. of \citet{hu2016dot}. In this work, we choose a third approach, which consists in doubling the number of input and output nodes.

\subsection{Analog Neural Networks vs Conventional Neural Networks}

The computations in analog neural networks are carried out in ways that are fundamentally different than in conventional deep learning.

To illustrate this, we consider the setting of Section \ref{sec:model}, in which the circuit is composed of input nodes (denoted $X$), internal nodes, and output nodes (denoted $\widehat{Y}$). The architecture and the components of the circuit (programmable resistors, diodes, amplifiers, etc.) determine the $X \mapsto \widehat{Y}$ function. In particular, the conductances of the programmable resistors (denoted $\theta$) parameterize this function, which we can write in the form $\widehat{Y} = f(X,\theta)$.

In conventional deep learning, the function $\widehat{Y} = f(X,\theta)$ implemented by a neural network is \textit{explicitly} implemented as a composition of elementary operations (such as tensor multiplications and nonlinear activation functions). In contrast, one way to think of an analog neural network, is that its state is described by its node voltages, and that the function $\widehat{Y} = f(X,\theta)$ is \textit{implicitly} determined by Kirchhoff's current law.

\subsubsection{Network Architecture vs Hardware Architecture}

Computer architectures for conventional deep learning involve several layers of abstractions which separate the neural network implementation (at a software level) from the underlying device physics (at a hardware level). This enables deep learning practitioners to design their neural networks without necessarily understanding the hardware design, and vice-versa. This has led software engineers and hardware engineers to adopt somewhat different languages. In particular, the term `architecture' is used both in hardware design and neural network design, but with very different meanings.

In contrast, by implementing the neural network directly at the hardware level, our framework for end-to-end analog neural networks eliminates all intermediate abstractions between the hardware and the neural network and thus blurs the distinction between `hardware architecture' and `network architecture'. This `architecture' implements and determines the $\widehat{Y} = f(X,\theta)$ input-to-output function of the circuit.

\clearpage
\section{SPICE Simulation Details (Complement of Section \ref{sec:numerical-simulations})}
\label{sec:simulation-details}

In this appendix, we provide details about our numerical simulations with SPICE.
First we describe the general training procedure (section \ref{sec:netlist-details}), and then we describe the specific details related to solving the XOR task (section \ref{sec:xor-details}) and the MNIST task (section \ref{sec:mnist-details}).

\subsection{Simulation Setup}
\label{sec:netlist-details}

In our simulations, we use SPICE to perform the first phase (free phase) as well as the second phase (nudged phase) of EqProp. The other operations are performed in Python: this includes weight initialization and netlist generation (before training starts), data normalization (before the first phase of EqProp), calculating loss and gradient currents (between the first and second phases of EqProp), weight gradient calculation (at the end of the second phase of EqProp) and performing the weight updates (we update resistances in software).

\paragraph{Weight Initialization and Netlist Generation.}

In SPICE, a circuit is defined by a \textit{netlist}, which is a text-based representation of the circuit. 
For an analog network with $n_\text{in}$ input nodes, one hidden layer of $n_\text{hidden}$ neurons, and $n_\text{out}$ output nodes, the corresponding netlist is created by sequentially defining and linking the following components:

\begin{enumerate}
  \itemsep0em
  \item $n_\text{in}$ input voltage sources initialized at ground,
  \item $n_\text{in}$ nodes representing input units,
  \item a resistance matrix $\textbf{R}_1 \in \mathbb{R}^{n_\text{in} \times n_\text{hidden}}$,
  \item $n_\text{hidden}$ nodes representing the input nodes of hidden neurons,
  \item $n_\text{hidden}$ diodes, where each anode is connected to a hidden node and each cathode is initialized at ground,
  \item $n_\text{hidden}$ diodes, where each anode is initialized at ground and each cathode is connected to a hidden node,
  \item $n_\text{hidden}$ CCCS current amplifiers,
  \item $n_\text{hidden}$ VCVS current amplifiers,
  \item $n_\text{hidden}$ nodes representing the output nodes of hidden neurons (after nonlinear transfer function and amplification),
  \item a resistance matrix $\textbf{R}_2 \in \mathbb{R}^{n_\text{hidden} \times n_\text{out}}$,
  \item $n_\text{out}$ nodes representing output units,
  \item $n_\text{out}$ output current sources initialized at ground for current injection during the weakly clamped phase.
\end{enumerate}

The conductances (weights) are initialized by drawing i.i.d. samples uniformly at random in a range of values $[L,U]$. The values of $L$ and $U$ used for XOR and MNIST are provided in sections \ref{sec:xor-details} and \ref{sec:mnist-details} respectively. We note that SPICE does not support conductances, only resistances; therefore the weights of the network are stored as resistances ($\textbf{R}_1$ and $\textbf{R}_2$ here).

Steps 5 and 6 form the nonlinear transfer function and steps 7 and 8 form the signal amplification procedure. Steps 3 through 9 may be repeated to create deeper networks (assuming appropriately sized resistance matrices).

\paragraph{Training Iteration.}

Given a data set $\mathcal{D}$ of training samples $(X,Y)$, we optimize the network by stochastic gradient descent (SGD).
In order to perform one training iteration on a sample $(X,Y)$ (i.e. to compute the corresponding gradient and to take one step of SGD), we perform the following four steps below.

\subparagraph{Free Phase (Inference).}
We set the DC voltage of each input node $i \in \{ 1,2,\ldots,n_\text{in} \}$ to $X_i$ and set the current through each output node to zero. We then run the SPICE simulation under standard temperature conditions (25$^{\circ}$C). This calculates the (first) steady state of the network. We then extract the voltages 
at each hidden neuron (both before and after the transfer function) and output node. We denote these voltages $H_\text{in}^0$, $H_\text{out}^0$ and $\widehat{Y}^0$, respectively. Recall that, to address the constraint of nonnegative weights, the number of output nodes are doubled, so that $\widehat{Y}^0 = \{ \widehat{Y}_k^{+,0}, \widehat{Y}_k^{-,0} \}_{1 \leq k \leq K}$, with $\widehat{Y}_k^{+,0} - \widehat{Y}_k^{-,0}$ serving as prediction for class $k$ ($1 \leq k \leq K$).

\subparagraph{Loss Calculation and Gradient Currents.}
We then calculate in software the squared error loss:
\begin{equation}
    \mathcal{L} = \sum_k \left( Y_k - \widehat{Y}_k^{+,0}+\widehat{Y}_k^{-,0} \right)^2,
\end{equation}
as well as the gradient currents $I_k^+$ and $I_k^-$ to source at output nodes $\widehat{Y}_k^+$ and $\widehat{Y}_k^-$:
\begin{equation}
    I_k^+ = \beta \; (Y_k-\widehat{Y}_k^{+,0}+\widehat{Y}_k^{-,0}), \qquad I_k^- = \beta \; (\widehat{Y}_k^{+,0}-\widehat{Y}_k^{-,0}-Y_k),
\end{equation}
where $\beta$ is a hyperparameter that has the physical dimensions of a conductance.

\subparagraph{Nudged Phase.}
We set the current at each output nodes $\widehat{Y}_k^+$ and $\widehat{Y}_k^-$ to $I_k^+$ and $I_k^-$, respectively. We then run the simulation under the same conditions as in the free phase. This simulation produces the second steady state of the network. We extract anew the voltages at hidden and output nodes, denoted $H_\text{in}^\beta$, $H_\text{out}^\beta$ and $\widehat{Y}^\beta$, respectively.

\subparagraph{Weight Updates.}
Using the listed voltages, we calculate the voltage drops in software. For the first layer, we calculate $\Delta V_{1}^{0} = X \ominus H_\text{in}^{0}$ and $\Delta V_{1}^{\beta} = X \ominus H_\text{in}^{\beta}$, where $\ominus$ denotes the outer subtraction operator. Similarly for the second layer, we calculate $\Delta V_{2}^{0} = H_\text{out}^{0} \ominus \widehat{Y}^{0}$ and $\Delta V_{2}^{\beta} = H_\text{out}^{\beta} \ominus \widehat{Y}^{\beta}$. We then read the resistances from the netlist, convert them to conductances which we gather into matrices of conductances $\textbf{G}_1$ and $\textbf{G}_2$. We then update these matrices according to Theorem \ref{thm:gradients}:
\begin{align}
\label{eq:simulated-update1}
\textbf{G}_1 & \leftarrow \textbf{G}_1 - \frac{\alpha_1}{\beta} \left [ \left ( \Delta V_{1}^{\beta} \right )^2 - \left (\Delta V_{1}^{0} \right)^2 \right ], \\
\label{eq:simulated-update2}
\textbf{G}_2 & \leftarrow \textbf{G}_2 - \frac{\alpha_2}{\beta} \left [ \left ( \Delta V_{2}^{\beta} \right )^2 - \left (\Delta V_{2}^{0} \right)^2 \right ],
\end{align}
where $\alpha_1$ and $\alpha_2$ are learning rates for the first and second layers, respectively. All conductances below a threshold level $L$ are then clipped to $L$, to account for the fact that real-world conductances are strictly positive. We then convert these conductance matrices ($\textbf{G}_1$ and $\textbf{G}_2$) back to resistance matrices ($\textbf{R}_1$ and $\textbf{R}_2$), and we update the netlist accordingly.


\paragraph{Mini-batch Updates.} Updating the resistance matrices in software with SPICE is time consuming, as this framework was not designed for such purpose. To limit the time wasted in these updates, we perform mini-batch gradient descent with mini-batches of size $m = 100$, i.e. each weight update corresponds to the sum of the gradients of $m$ samples. However, contrary to conventional mini-batch processing in deep learning (where all $m$ gradients would be computed in parallel), in our setting the gradients are computed one at a time. Indeed, data samples must be processed one at a time in our analog neural network. For each sample within a mini-batch, we sequentially perform the first phase (free phase) and second phase (nudged phase) to compute the corresponding gradients, sans weight update. The gradients are stored\footnote{We note that it is possible to perform approximate mini-batch updates on analog crossbar arrays without storing the gradients \citep{hoskins2019sbe}. In this case the mini-batch update corresponds to a rank-1 approximation of the gradient for a given mini-batch.}.
After processing all samples of the mini-batch, we perform the weight updates corresponding to the sum of the $m$ gradients. The weight updates of Eq.~\ref{eq:simulated-update1}-\ref{eq:simulated-update2} adapted for mini-batch updates are:
\begin{align}
\textbf{G}_1 & \leftarrow \textbf{G}_1 - \frac{\alpha_1}{m \beta} \sum_{i=1}^{m}{ \left [ \left ( \Delta {V_{i,1}^{\beta}} \right )^2 - \left (\Delta {V_{i,1}^{0}} \right)^2 \right ]}, \\
\textbf{G}_2 & \leftarrow \textbf{G}_2 - \frac{\alpha_2}{m \beta} \sum_{i=1}^{m}{ \left [ \left ( \Delta {V_{i,2}^{\beta}} \right )^2 - \left (\Delta {V_{i,2}^{0}} \right)^2 \right ]},
\end{align}
where the index $i$ refers to the data sample $i$.

\subsection{XOR Task}
\label{sec:xor-details}

The XOR task consists in training a function $\widehat{Y} = f(X_1,X_2,\theta)$ to produce $\widehat{Y} = X_1 \; \textrm{XOR} \; X_2$. The corresponding inputs-target pairs are given by the following table.

\begin{center}
 \begin{tabular}{||c c||} 
 \hline
 $X$ & $Y$ \\ [0.5ex] 
 \hline\hline
 $\langle 0, 0 \rangle$ & $\langle 0 \rangle$ \\ 
 \hline
 $\langle 0, 1 \rangle$ & $\langle 1 \rangle$ \\
 \hline
 $\langle 1, 0 \rangle$ & $\langle 1 \rangle$ \\
 \hline
 $\langle 1, 1 \rangle$ & $\langle 0 \rangle$ \\ [1ex] 
 \hline
\end{tabular}
\end{center}

We find it necessary to shift and scale the input values. Instead of using the values $0$ and $1$, we use $-2$ and $2$. Hence, our training dataset is as follows:

\begin{center}
 \begin{tabular}{||c c||} 
 \hline
 $X$ & $Y$ \\ [0.5ex] 
 \hline\hline
 $\langle -2, -2 \rangle$ & $\langle 0 \rangle$ \\ 
 \hline
 $\langle -2, 2 \rangle$ & $\langle 1 \rangle$ \\
 \hline
 $\langle 2, -2 \rangle$ & $\langle 1 \rangle$ \\
 \hline
 $\langle 2, 2 \rangle$ & $\langle 0 \rangle$ \\ [1ex] 
 \hline
\end{tabular}
\end{center}

The network architecture used is represented in Figure.~\ref{fig:xor}.

The conductances (weights) are initialized by drawing i.i.d. samples uniformly at random in the range $[L, U]$ with $L = 0.0001$ S and $U=0.1$ S. In the second phase (nudged phase), we scale the gradient currents by a factor $\beta = 0.001$. We use $\alpha = 0.001$ as the learning rate (for all conductances).
During the weight update phase, all conductances that fall below the threshold level $L=10^{-7}$ S are clipped to $10^{-7}$ S to account for the fact that real-world conductances are strictly positive.
We train the network for 1000 iterations.

This training procedure produces a network that is able to perform the XOR operation. See Fig.~\ref{fig:xor} for the final network after completion of training. Since this analog network is very small and only requires 1000 training iterations, the simulation can run in under a minute on a standard computer.

\begin{figure*}[ht!]
\begin{center}
\includegraphics[width=\textwidth]{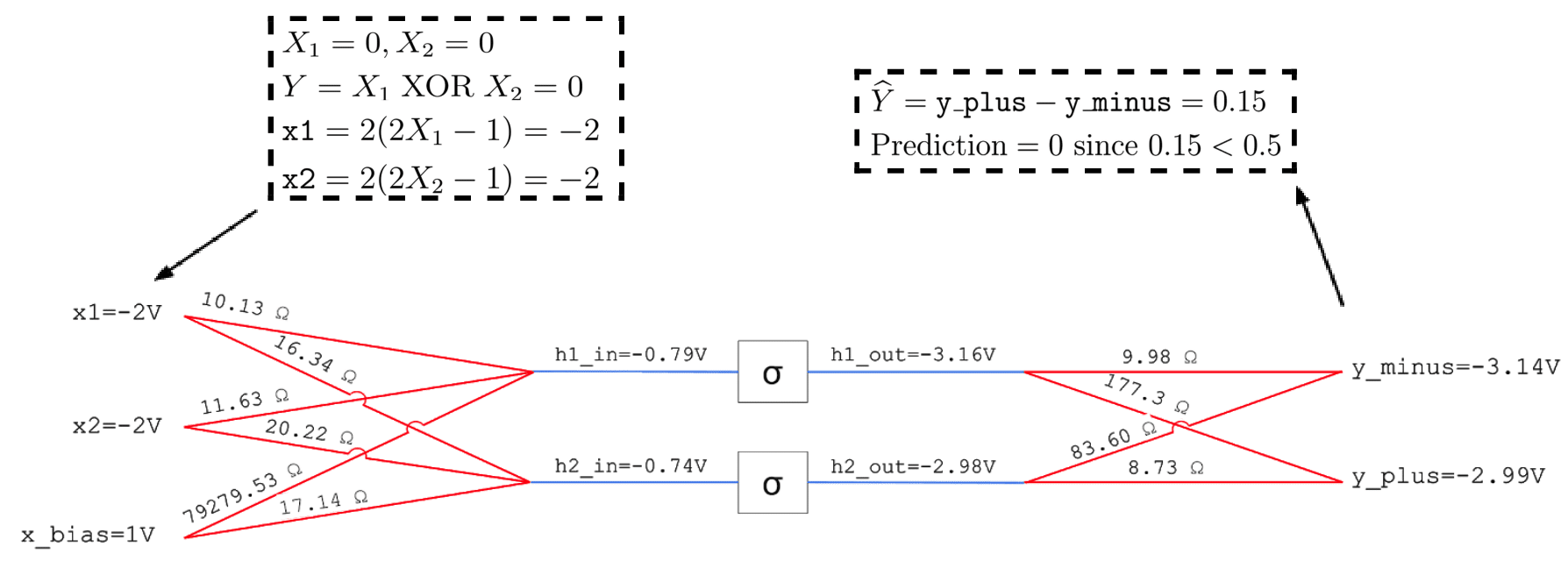}
\end{center}
\caption{SPICE network used to solve the XOR task, together with the final weights after training. The network has two input nodes ($X_1$ and $X_2$), plus a node whose voltage is always set to the same value of $X_{\rm bias} = 1 V$ which serves as a bias for the two hidden neurons. The symbol $\sigma$ denotes the antiparallel diodes and the bidirectional amplifier which implement the nonlinear transfer function (as in Figure~\ref{fig:network}). The gains of the amplifiers are set to $A = 4$. To overcome the constraint of non-negative weights (non-negative conductances), our network has two output nodes $\widehat{Y}_+$ and $\widehat{Y}_-$, with the prediction being $\widehat{Y} = \widehat{Y}_+ - \widehat{Y}_-$. Also, note that SPICE does not support conductances; therefore the weights are represented as resistances.}
\label{fig:xor}
\end{figure*}

\subsection{MNIST Classification Task}
\label{sec:mnist-details}

We now give details about our simulations on the MNIST dataset (the `modified' version of the National Institute of Standards and Technology dataset) \citep{lecun1998gradient}.
The MNIST dataset consists of 32x32 gray-scaled images representing digits, together with their class label (a digit between 0 and 9). It is composed of 50,000 training samples and 10,000 test samples.

To overcome the constraint of non-negative weights (non-negative conductances), we double the number of output nodes. We also double the number of input nodes and invert one set. Thus, our network has $1568$ input nodes (two nodes per pixel of a $28 \times 28$ image), one hidden layer of $100$ neurons, and $20$ output nodes (two nodes per each of the $10$ digit classes). We also use a `node' $X_{\rm bias}$ whose voltage is always set to the same value of $1 V$, which serves as a bias for the hidden neurons.
This gives us $X_1$, $-X_1$, $X_2$, $-X_2$, ..., $X_{784}$, $-X_{784}$ and $X_{\rm bias}$ as input nodes, $H_1$, $H_2$, ..., $H_{100}$ as hidden neurons, and $\widehat{Y}_0^+$, $\widehat{Y}_0^-$, $\widehat{Y}_1^+$, $\widehat{Y}_1^-$, ..., $\widehat{Y}_{9}^+$, $\widehat{Y}_{9}^-$ as output nodes. The prediction of the model is by definition
\begin{equation}
    Y_{\rm pred} = \underset{0 \leq i \leq 9}{\arg \max} \left( \widehat{Y}_i^+ - \widehat{Y}_i^- \right).
\end{equation}

For each weight matrix, the conductances are initialized by drawing i.i.d. samples uniformly at random in the range $[L, U]$ with $L=10^{-7}$ and $U = \frac{0.08}{\sqrt{n_i + n_{i+1}}}$, where $n_i$ is the fan-in and $n_{i+1}$ is the fan-out of the weight matrix\footnote{We found that the Glorot initialization scheme \citep{glorot2010understanding} was not optimal and that using the scaling factor 0.08 in the upper bound $U$ yields better results. Future work should investigate better initialization schemes for analog neural networks.}. The gains of the amplifiers are set to $A = 4$. At each training iteration, we normalize each image sample to have mean $0$ and standard deviation $5$. In the second phase (nudged phase), we scale the gradient currents by a factor $\beta$ with $|\beta| = 0.01$. We choose the sign of $\beta$ at random for each training iteration\footnote{This technique, which is also used by \citet{Scellier+Bengio-frontiers2017} and \citet{ernoult2020equilibrium}, helps improve the test accuracy.}. We use $\alpha_1 = 0.1$, and $\alpha_2 = 0.05$ as learning rates for the first and second weight matrices.
During the weight update phase, all conductances that fall below the threshold level $L=10^{-7}$ S are clipped to $10^{-7}$ S to account for the fact that real-world conductances are strictly positive.
Since writing and loading new resistances for each training iteration is costly (in software simulations), we save simulation time by performing the weight updates every $100$ iterations, by storing intermediate cumulative gradients. The resulting weight updates are thus equivalent to those of mini-batch gradient descent with minibatches of size $m=100$ (although data samples are processed one at a time).

\subsubsection{Logistic Regression Classifier (Benchmark)}

To demonstrate that our analog neural network (SPICE-based network) benefits from the nonlinearity of its devices, we show that it outperforms a logistic regression model. The logistic regression model is a linear transformation of the inputs, followed by a softmax and the cross-entropy loss. For a fair comparison with our SPICE-based network, we double the number of input nodes (we use two nodes per pixel and invert one set), and each image sample is normalized to have mean $0$ and standard deviation $1$. Thus the logistic regression model has $2 \times 784 = 1568$ input units and $10$ output units. The test error rate obtained is $7.27\%$, which is significantly higher than the test error rate of our SPICE implementation ($3.43\%$). \citet{lecun1998gradient} also report results with different kinds of linear classifiers (corresponding to different pre-processing methods), all performing significantly worse than our SPICE model.

\subsubsection{PyTorch Implementation of EqProp (Benchmark)}

We benchmark our SPICE-implementation of EqProp against a PyTorch implementation of the original EqProp model \citep{Scellier+Bengio-frontiers2017}. We use two models for benchmarking: a \textit{standard model} in which the weights are free to be either positive or negative, and a \textit{positive model} with positive weights. For the positive model, we double the input units (and invert one set) and output units as is the case in the SPICE implementation.

\paragraph{Standard model.}
In the standard model, we use the Glorot initialization scheme \citep{glorot2010understanding}, i.e. each weight matrix is initialized by drawing i.i.d. samples uniformly at random in the range $[L, U]$, where $L=-\frac{\sqrt{6}}{\sqrt{n_i + n_{i+1}}}$ and $U=\frac{\sqrt{6}}{\sqrt{n_i + n_{i+1}}}$, with $n_i$ the fan-in and $n_{i+1}$ the fan-out of the weight matrix. In the second phase (nudged phase), we use a scaling factor $\beta$ with $|\beta| = 0.01$. We choose the sign of $\beta$ at random for each training iteration. We use $\alpha_1 = 0.1$ and $\alpha_2 = 0.05$ as learning rates for the first and second weight matrices. We use $500$ and $100$ iterations to approximate the steady states during the free phase and nudged phase, respectively, with a step size of $\epsilon=0.02$. We use minibatches of size $100$.

\paragraph{Positive model.}
In the positive model, each weight matrix is initialized by drawing i.i.d. samples uniformly at random in the range $[L, U]$, where $L=10^{-7}$ and $U = \frac{0.08}{\sqrt{n_i + n_{i+1}}}$. In the second phase (nudged phase), we use a scaling factor $\beta$ with $|\beta| = 0.01$. We choose the sign of $\beta$ at random for each training iteration. We use $\alpha_1 = 0.2$ and $\alpha_2 = 0.1$. During the weight update phase, all weights that fall below the threshold level $L=10^{-7}$ are clipped to $10^{-7}$. We use $500$ and $100$ iterations to approximate the steady states during the free phase and nudged phase, respectively, with a step size of $\epsilon=0.02$. We use minibatches of size $100$.

\begin{figure*}[ht!]
\begin{center}
\includegraphics[width=\textwidth]{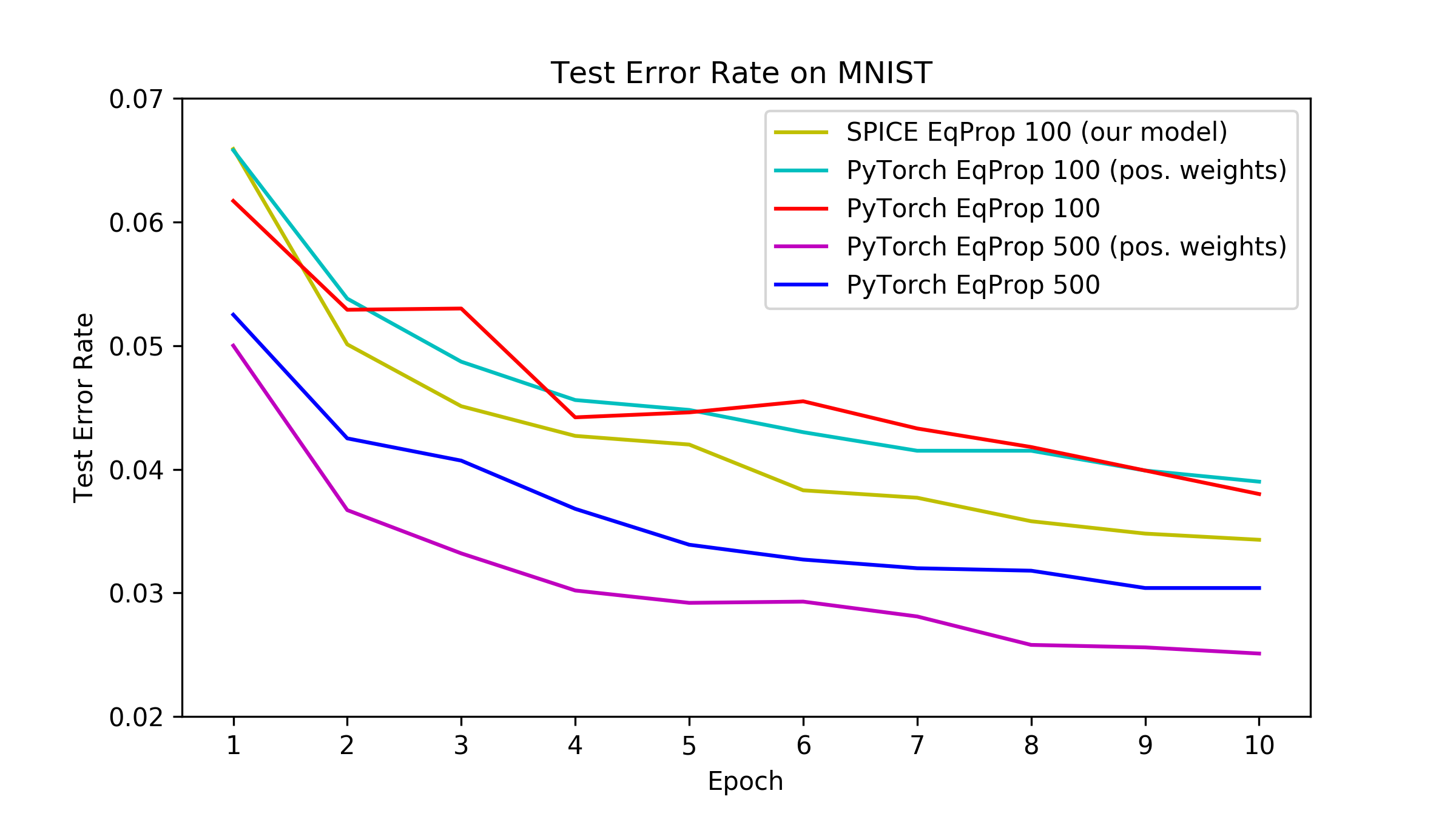}
\end{center}
\caption{Training results on MNIST. SPICE EqProp (our model) is benchmarked against PyTorch EqProp, a PyTorch implementation of the original EqProp model \citep{Scellier+Bengio-frontiers2017}. 100 and 500 are the numbers of hidden neurons. `pos. weights' means that the weights are constrained to be positive.}
\label{fig:training-curves}
\end{figure*}

\end{document}